\documentclass[10pt]{article}
\usepackage[left=1in,top=1in,right=1in,bottom=1in,letterpaper]{geometry}

\usepackage{amssymb,amsmath,enumerate}
\usepackage{latexsym,amsfonts,amscd,amsxtra,amstext}
\usepackage{hyperref}
\hypersetup{colorlinks=true, linkcolor=blue, citecolor=blue, urlcolor=blue, pdftitle={Lower Bounds and Nearly Optimal Algorithms in Distributed Learning with Communication Compression}, pdfauthor={Xinmeng Huang, Yiming Chen, Wotao Yin, Kun Yuan}}

\usepackage[pdftex]{graphicx}
\usepackage{amsmath}
\usepackage{amsthm}
\usepackage{amssymb}
\usepackage{empheq}
\usepackage{dsfont}
\usepackage{booktabs}
\usepackage{color}
\usepackage{xcolor}
\usepackage{subcaption}

\usepackage{multirow}
\usepackage{makecell}
\usepackage{wrapfig}

\usepackage{algorithmic}
\usepackage{algorithm}

\def\PP{\mathbb{P}}
\def\NN{\mathbb{N}}

\def\bsv{{\boldsymbol{v}}}

\def\prog{\mathrm{prog}}

\def\cA{{\mathcal{A}}}

\def\cC{{\mathcal{C}}}
\def\cF{{\mathcal{F}}}
\def\cN{{\mathcal{N}}}
\def\cO{{\mathcal{O}}}

\def\cU{{\mathcal{U}}}
\def\cV{{\mathcal{V}}}
\def\cX{{\mathcal{X}}}
\def\cY{{\mathcal{Y}}}

\newcommand{\bx}{\mathbf{x}}
\newcommand{\by}{\mathbf{y}}
\newcommand{\be}{\mathbf{e}}
\newcommand{\bp}{\mathbf{p}}

\newcommand{\bsg}{\boldsymbol{g}}
\newcommand{\bsdelta}{\boldsymbol{\delta}}

\newcommand{\EE}{\mathbb{E}}
\newcommand{\RR}{\mathbb{R}}

\newcommand{\er}{\mathrm{er}}

\usepackage{xspace}
\newcommand{\eg}{e.g.\xspace}
\newcommand{\ie}{i.e.\xspace}

\newcommand{\psgd}{P-SGD\xspace}
\newcommand{\memsgd}{MEM-SGD\xspace}
\newcommand{\ds}{Double-Squeeze\xspace}
\newcommand{\ef}{EF21-SGD\xspace}
\newcommand{\ours}{NEOLITHIC\xspace}
\newcommand{\cifar}{CIFAR-10\xspace}

\newcommand{\tablefontsize}{\scriptsize}

\newtheorem{theorem}{Theorem}
\newtheorem{assumption}{Assumption}
\newtheorem{remark}{Remark}
\newtheorem{lemma}{Lemma}
 
\newtheorem{corollary}{Corollary}
\newtheorem{definition}{Definition}
 
\title{Lower Bounds and Nearly Optimal Algorithms in Distributed Learning with Communication Compression}

\author{Xinmeng Huang\footnote{Equal Contribution.} \footnote{Graduate Group in Applied Mathematics and Computational Science, Univ. of Pennsylvania. \texttt{xinmengh@sas.upenn.edu}.}\,
    \and Yiming Chen\footnotemark[1]\,\,\footnote{DAMO Academy,
			Alibaba Group. \texttt{\{charles.cym, wotao.yin, kun.yuan\}@alibaba-inc.com}.}
    \and Wotao Yin\footnotemark[3] 
    \and Kun Yuan\footnotemark[3]\,\,\footnote{Corresponding Author.}
}

\begin{document}
\maketitle
\begin{abstract}
Recent advances in distributed optimization and learning have shown that communication compression is one of the most effective means of reducing communication. While there have been many results on  convergence rates under communication compression, a theoretical lower bound is still missing.

Analyses of algorithms with communication compression have attributed convergence to two abstract properties: the unbiased property or the contractive property. They can be applied with either unidirectional compression (only messages from workers to server are compressed) or bidirectional compression. In this paper, we consider distributed stochastic algorithms for minimizing smooth and non-convex objective functions under communication compression. We establish a convergence lower bound for algorithms whether using unbiased or contractive compressors in unidirection or bidirection. To close the gap between  the lower bound and the existing upper bounds, we further propose an algorithm, NEOLITHIC, which almost reaches our lower bound (up to logarithm factors) under mild conditions. Our results also show that using contractive bidirectional compression can yield iterative methods that converge as fast as those using unbiased unidirectional compression. The experimental results  validate our findings.
\end{abstract}

\section{Introduction}
Large-scale optimization is a critical step in many machine learning applications. Millions or even billions of data samples contribute to the excellent performance in tasks such as robotics, computer vision, natural language processing, healthcare, and so on. However, such a scale of data samples and model parameters leads to enormous communication that hampers the scalability of distributed machine-learning training systems. We urgently need communication-reduction strategies. State-of-the-art strategies include communication compression \cite{Alistarh2017QSGDCS, Bernstein2018signSGDCO, Seide20141bitSG}, decentralized communication \cite{Lian2017CanDA, assran2019stochastic, chen2021accelerating}, lazy communication \cite{stich2019local,karimireddy2020scaffold,chen2018lag}, and beyond. This article will focus on the former.

The most common method of distributed training is Parallel SGD (P-SGD) \cite{Bottou2010LargeScaleML}. In P-SGD, the stochastic gradients that workers transmit to a server cause significant communication overhead in large-scale machine learning. To reduce this overhead, many recent works propose to compress the messages sent unidirectionally from workers to server \cite{Alistarh2018TheCO, Horvath2019NaturalCF, Alistarh2017QSGDCS} or compress the messages between them bidirectionally \cite{Tang2019DoubleSqueezePS,Xu2021StepAheadEF}. The method of compression is either sparsification or quantization \cite{Alistarh2018TheCO, Horvath2019NaturalCF,Alistarh2017QSGDCS} or their combination \cite{Horvath2019NaturalCF, Beznosikov2020OnBC}. The literature \cite{Stich2018SparsifiedSW, Sahu2021RethinkingGS, Tang2019DoubleSqueezePS, Xu2021StepAheadEF} reveals that bidirectional compression can save more communication, but leads to slower convergence rates.

Although there are many specific compression methods, their convergence analyses are mainly built on unbiased compressibility or contractive compressibility. The literature \cite{Beznosikov2020OnBC, Safaryan2020UncertaintyPF, Xu2020CompressedCF} summarizes these two properties and how they appear in the analyses. An unbiased compressor compresses a vector $x$ into a quantized version $C(x)$ satisfying $\mathbb{E}[C(x)] = x$, \ie, no bias is introduced. The contractive compressor may introduce bias, but its compression introduces much less variance. We give their definitions below. Although the contractive compressor can empirically work better, the analysis based on the unbiased compressor yields faster convergence due to unbiasedness \cite{Beznosikov2020OnBC, Mishchenko2019DistributedLW, Horvath2019StochasticDL,Gorbunov2021MARINAFN}.

Despite the quick progress made in compression techniques and their convergence, we do not yet understand the limits of algorithms with communication compression. Since unbiased and contractive compressibilities
are the two representative characteristics, 
we use them to generally theorize two types of compressors. For each type, we intend to answer:

\emph{What is the optimal convergence rate that a distributed algorithm can achieve when using any compressor of this type?} 

Here, we assume that only unbiased compressibility or contractive compressibility can be utilized, not considering any additional special compressor design; after all, any special design in the literature has been heuristic, whose effectiveness can be explained at best and not proved or quantified. So, we further clarify our question:
\emph{Given a class of optimization problems (specified below) and a type of compressors, if we choose the worst combination of them to defeat an algorithm, what will the convergence rate that the best-defending algorithm can reach?} To our knowledge, they are fundamental questions not addressed yet.

\begin{table}[h!]
	\caption{\small Comparison between various distributed stochastic algorithms with communication compression for non-convex loss functions. For clarity, we keep the stochastic gradient variance $\sigma^2$, data heterogeneity bound $b^2$
	, gradient square norm bound $G^2$ (used in \cite{Stich2018SparsifiedSW,Tang2019DoubleSqueezePS,Xie2020CSERCS}, much larger than $b^2$ and $L$)
	, but omit
	smoothness constant $L$, and initialization $f(x^{(0)})-f^\star$ in the below results. Moreover, notation $\tilde{O}(\cdot)$ hides all logarithmic terms. Notation $\delta$ and $\omega$ are compressor-related parameters (see detailed discussions in Section~\ref{sec:setup}). 
	}
	\footnotesize
	\centering
	\begin{tabular}{c  l l  c  l}
		\toprule
		&\textbf{Algorithm} & \textbf{Convergence Rate} &  \textbf{Compression}$^\mathrm{a}$ & \textbf{Trans. Compl.}$^\mathrm{b}$\\ \midrule
		\multirow{2}{*}{Lower Bound}&{\color{blue}\makecell[c]{Theorem \ref{thm:lower1}+Corollary \ref{cor:lower2}}} & \textcolor{blue}{$\Omega\left(\frac{\sigma}{\sqrt{nT}}+\frac{1}{\delta T}\right)$} & \textcolor{blue}{\makecell[c]{Uni/Bidirectional\\Contractive}} & \textcolor{blue}{$\Omega(n/\delta^2)$}  \\ 
		&{\color{blue}\makecell[c]{Theorem \ref{thm:lower3}+Corollary \ref{cor:lower4}}} & \textcolor{blue}{$\Omega\left(\frac{\sigma}{\sqrt{nT}}+\frac{1+\omega}{T}\right)$} & \textcolor{blue}{\makecell[c]{Uni/Bidirectional\\Unbiased}} & \textcolor{blue}{$\Omega\left(n(1+\omega)^2\right)$} \\ \midrule
		\multirow{7}{*}{Upper Bound}&{\color{blue}Theorem \ref{thm:convergence-rate}} & \textcolor{blue}{$\tilde{O}\left(\frac{\sigma}{\sqrt{nT}}+\frac{1}{\delta T}\right)$} & \textcolor{blue}{\makecell[c]{Uni/Bidirectional\\Contractive}} & \textcolor{blue}{$\tilde{O}(n/\delta^2)$}   \\ 
		&{\color{blue}Corollary \ref{cor:convergence-rate-unbiased}} & \textcolor{blue}{$\tilde{O}\left(\frac{\sigma}{\sqrt{nT}}+\frac{1+\omega}{T}\right)$} & \textcolor{blue}{\makecell[c]{Uni/Bidirectional\\Unbiased}} & \textcolor{blue}{$\tilde{O}(n(1+\omega)^2)$}   \\
		&Q-SGD \cite{Jiang2018ALS} & $O\left(\frac{(1+\omega)^{1/2}\sigma+\omega^{1/2} b}{\sqrt{nT}}\right)^{\dagger,\diamond}$ & \makecell[c]{Unidirectional\\i.i.d, Unbiased}  & $-$ \\ 
		&MEM-SGD\,\cite{Stich2018SparsifiedSW} & $O\left(\frac{\sigma}{\sqrt{nT}}+\frac{G^{2/3}}{\delta ^{2/3}T^{2/3}}+\frac{1}{T}\right)^\diamond$ & \makecell[c]{Unidirectional\\Contractive} & $O(n^3/\delta^4)$   \\ 
		&Double-Squeeze\,\cite{Tang2019DoubleSqueezePS} & $O\left(\frac{\sigma}{\sqrt{nT}}+\frac{G^{2/3}}{\delta ^{4/3}T^{2/3}}+\frac{1}{T}\right)^\diamond$ & \makecell[c]{Bidirectional\\Contractive} & $O(n^3/\delta^8)$  \\
		&CSER\,\cite{Xie2020CSERCS} & $O\left(\frac{\sigma}{\sqrt{nT}}+\frac{G^{2/3}}{\delta ^{2/3}T^{2/3}}+\frac{1}{T}\right)^\diamond$ & \makecell[c]{Unidirectional\\Contractive} & $O(n^3/\delta^4)$  \\ 
		&EF21-SGD\,\cite{Fatkhullin2021EF21WB} & $O\left(\frac{\sigma}{\sqrt{\delta^3 T}}+\frac{1}{\delta T}\right)^\ast$ & \makecell[c]{Unidirectional\\Contractive} & $-$    \\
		\bottomrule
		\multicolumn{5}{l}{$^\mathrm{a}$\,\footnotesize{This column indicates the type of the compressor and in what direction the compression is applied.}}\\
		\multicolumn{5}{l}{$^\mathrm{b}$\,\footnotesize{This column indicates the transient complexity, \emph{\ie}, the number of gradient queries (or communication rounds)}}\\
		\multicolumn{5}{l}{\;\;\,\footnotesize{the algorithm has to experience before reaching the linear-speedup stage, \emph{\ie}, $O(\frac{\sigma}{\sqrt{nT}})$.}}  \\
		\multicolumn{5}{l}{$^\dagger$\,\footnotesize{This convergence rate is valid only for $T=\Omega(n(1+\frac{\omega}{n})^2)$. Since the rate $\frac{(1+\omega)^{1/2}\sigma+\omega^{1/2} b}{\sqrt{nT}}$ is always worse than $\frac{\sigma}{\sqrt{nT}}$,}}\\
		\multicolumn{5}{l}{\;\;\footnotesize{the transient complexity is not available.}}\\
		\multicolumn{5}{l}{$^\ast$\,\footnotesize{Since the convergence rate does not show linear-speedup, the transient complexity is not available}}\\
		\multicolumn{5}{l}{$^\diamond$\,{The rates are either strengthened by relaxing the original restrictive assumptions or extended to the same setting as}}\\
		\multicolumn{5}{l}{\;\;{NEOLITHIC for a fair comparison, see more details in Appendix \ref{app:more-details}}.}
% 		\multicolumn{5}{l}{$^\diamond$\,{The rates are either strengthened by relaxing the original restrictive assumptions or extended to the same}}\\
% 		\multicolumn{5}{l}{\;\;{setting as NEOLITHIC for a fair comparison, see more details in Appendix \ref{app:more-details}}.}
	\end{tabular}
	\label{tab:introduction-comparison}
\end{table}

\subsection{Main Results}
This paper clarifies these  open questions by providing lower bounds under the non-convex smooth stochastic optimization setting, and developing effective algorithms that match the lower bounds up to logarithm factors. In particular, our contributions are: 
\begin{itemize}
	\item We establish convergence lower bounds for distributed algorithms with communication compression in the stochastic non-convex regime. Our lower bounds apply to any algorithm conducting unidirectional or bidirectional compression with  unbiased or contractive compressors. 
We find a clear gap between the established lower bounds and the existing convergence rates.
	\item We propose a novel \underline{\textbf{ne}}arly \underline{\textbf{o}}ptimal a\underline{\textbf{l}}gor\underline{\textbf{ith}}m w\underline{\textbf{i}}th \underline{\textbf{c}}ompression (NEOLITHIC) to fill in this gap. NEOLITHIC can adopt either  unidirectional or bidirectional compression, and is compatible with both unbiased and contractive compressors. 
	Using any combination, NEOLITHIC  provably matches the above lower bound, under an additional mild assumption and up to logarithmic factors. 
	
	\item The convergence results of NEOLITHIC imply that algorithms using biased contractive compressors bidirectionally can theoretically converge as fast as those with  unbiased compressors used unidirectionally (see discussion in Remark \ref{rmk:no-benefit}).
% 	Before our work, it is established in  \cite{Beznosikov2020OnBC,Stich2018SparsifiedSW,Sahu2021RethinkingGS,Tang2019DoubleSqueezePS,Xu2021StepAheadEF} that algorithms using unidirectional compression with unbiased compressors enjoy better convergence rates. 
    Before out work, it is only established in \cite{Philippenko2021PreservedCM} that, for convex problems and unbiased compressors, algorithms with bidirectional compression can match with the counterparts with unidirectional compression.
	
	\item We provide extensive experimental results to validate our theories. 
\end{itemize}

All established results in this paper as well as convergence rates of existing state-of-the-art distributed algorithms with communication compression are listed in Table \ref{tab:introduction-comparison}. The transient complexity, which measures how sensitive the algorithm is to the compression strategy (see Section~\ref{sec:lower-bound}), is also listed in the table. The smaller the transient complexity is, the faster the algorithm converges.  

\subsection{Related Works.}

\paragraph{Distributed learning.}
Distributed learning has been increasingly useful in training large-scale machine learning models \cite{Dean2012LargeSD}. It typically follows a centralized or decentralized setup.
Centralized approaches  \cite{Agarwal2012DistributedDS,Li2014ScalingDM}, with P-SGD as the representative, require a global averaging step per iteration in which all workers need to synchronize with a central server. 
Decentralized approaches, however, are based on neighborhood averaging in which each worker only needs to synchronize with its immediate neighbors. 
The local interaction between directly-connected workers can save remarkable communication overhead compared to the remote communication between a worker and the central server. 
Well-known decentralized algorithms include decentralized SGD \cite{nedic2009distributed,chen2012diffusion,Yuan2016OnTC,Yuan2022RevistOC}, D$^2$\cite{tang2018d,Yuan2021RemovingDH}, and stochastic gradient tracking   \cite{xin2020improved,koloskova2021improved,alghunaim2021unified}, and their momentum variants \cite{lin2021quasi,Yuan2021DecentLaMDM}. The lazy communication \cite{yu2019linear,stich2019local,chen2018lag} is also utilized to reduce communication overhead in which workers conduct multiple local updates before sending messages.
 Lazy communication is also widely used in federated learning \cite{mcmahan2017communication,stich2019local,karimireddy2020scaffold}. 
 
 \paragraph{Communication compression.}
To alleviate the communication overhead when transmitting the full model or stochastic gradient in distributed learning, communication compression is proposed in literature with two mainstream approaches: quantization and sparsification.
Quantization \cite{Jiang2018ALS,Tang2018CommunicationCF,Zhang2017ZipMLTL} is essentially an unbiased operator with random noise. For example, \cite{Seide20141bitSG} develop Sign-SGD by using only $1$ bit for each entry whose convergence is studied in  \cite{Bernstein2018signSGDCO,Bernstein2018signSGDWM,Wen2017TernGradTG}. Q-SGD \cite{Alistarh2017QSGDCS}
, as a generalized variant of Sign-SGD,
compresses each entry with more flexible bits and enables a trade-off between convergence rates and communication costs.
% Quantization operators can be embedded into Local-SGD to further reduce the communication overhead \cite{Jiang2018ALS}.
Sparsification, on the other hand, amounts to a biased but contractive operator. {\cite{Wangni2018GradientSF} suggest randomly dropping  entries to achieve a sparse  vector to communicate,} while  \cite{Stich2018SparsifiedSW} suggest to transmit a certain number of the  largest elements of the full model or gradient. 
The theories behind contractive compressors are limited to those in~\cite{Vogels2019PowerSGDPL,Lin2018DeepGC,Sun2019SparseGC} due to analysis challenges, and they are established with assumptions such as bounded gradients \cite{Zhao2019GlobalMC,Karimireddy2019ErrorFF} or quadratic loss functions \cite{Wu2018ErrorCQ}. More discussions on unbiased and biased compressors can be found in \cite{Beznosikov2020OnBC,Safaryan2020UncertaintyPF,Richtarik20223PCTP}. Communication compression can also be combined with other communication-saving techniques such as decentralization \cite{Liu2021LinearCD,Zhao2022BEERFO}.

\paragraph{Error compensation.}
The error compensation (feedback) mechanism is introduced by \cite{Seide20141bitSG} to mitigate the error caused by 1-bit quantization. \cite{Wu2018ErrorCQ} study SGD with error-compensated quantization for quadratic
functions with convergence guarantees. \cite{Stich2018SparsifiedSW} show that error compensation can reduce quantization-incurred errors 
for strongly convex loss functions in the single-node setting. 
However, their analysis is restricted to compressors with expectation compression error no larger than the magnitude of the input vector, which is not applicable to general contractive compressors, such as \cite{Bernstein2018signSGDCO}. 
Error-compensated SGD is studied in \cite{Alistarh2018TheCO} for non-convex loss functions with no establishment of an improved convergence rate. 
% The algorithm EF21~\cite{Richtrik2021EF21AN} applies to the deterministic setting, using contractive compressors unidirectionally. It can converge under very mild assumptions. It has been recently extended to the stochastic setting~\cite{Fatkhullin2021EF21WB}. We don't have a linear-speedup result for EF21 (yet).
A recent work \cite{Richtrik2021EF21AN} gives the first algorithm EF21, in the deterministic regime, using unidirectional compression with contractive compressors that  rely only on standard assumptions. The later work \cite{Fatkhullin2021EF21WB} extends EF21 to the stochastic regime but cannot show linear speedup in terms of the number of workers. 

\paragraph{Lower bounds in optimization.} 
Lower bounds are well studied in convex optimization \cite{Fang2018SPIDERNN,Agarwal2015ALB,Diakonikolas2019LowerBF} especially when there is no gradient stochasticity
\cite{Arjevani2015CommunicationCO,Nesterov2004Intro,Balkanski2018ParallelizationDN,AllenZhu2018HowTM,Foster2019TheCO}. In non-convex optimization, \cite{Carmon2020LowerBF,Carmon2021LowerBF}
propose a zero-chain model and show a tight bound for first-order methods. By splitting the zero-chain model into multiple components, \cite{Zhou2019LowerBF,Arjevani2019LowerBF} extend the approach to finite sum and stochastic problems. Recently, \cite{Lu2021OptimalCI,Yuan2022RevistOC} show the lower bounds in the decentralized stochastic setting by assigning disjoint components of the model to remote nodes in the graphs. 
For distributed learning with communication compression, a useful work \cite{Philippenko2020BidirectionalCI} establishes a lower bound for unidirectional/bidirectional compression for strongly convex problems with fixed learning rates and unbiased compressors. There are limited studies on lower bounds for non-convex distributed learning with unbiased/contractive compressions.
% However, no lower bound result exists for distributed learning with communication compression to our knowledge.
% However, for algorithms conducting communication compression, there is only one existing lower bound \cite{Philippenko2020BidirectionalCI} restricted to unbiased compressors, strongly convex problems, and fixed learning rates. 
% Regardless of its different setting, the lower bound is limited in utility since stochastic algorithms typically use decaying learning rates. Besides, the optimality of the lower bound is not fully justified.

\section{Problem Setup}\label{sec:setup}
In this section, we introduce the notations and assumptions used throughout the paper. 
We consider standard  distributed learning with $n$ parallel workers. The data in worker $i$ follow  a local distribution $D_i$, which can be heterogeneous among all workers. These workers, together with a central parameter server, collaborate to train a model by solving
\begin{equation}\label{eqn:prob}
\min _{x \in \mathbb{R}^{d}}\quad f(x)=\frac{1}{n} \sum_{i=1}^{n} f_{i}(x) \quad\text{with}\quad f_{i}(x)=\mathbb{E}_{\xi_{i} \sim D_{i}} [F(x ; \xi_{i})],
\end{equation}
where  $F(x ;\xi)$ is the loss function evaluated at parameter $x$ with datapoint $\xi$. Since the objective $f$ can be non-convex, finding a global minimum of \eqref{eqn:prob} is generally intractable. Therefore, we turn to seeking a model $\hat{x}$ with a small gradient magnitude in expectation, \emph{\ie}, $\EE[\|\nabla f(\hat{x})\|^2]$. Next we introduce the setup under which we study the convergence rate.

\paragraph{Function class. } We let the {function class $\cF_{\Delta, L}$} denote the set of all functions satisfying Assumption \ref{asp:smooth} for any underlying dimension $d\in \NN_+$ and a given initialization point $x^{(0)}\in\RR^d$.
\begin{assumption}[\bf Smoothness] \label{asp:smooth}
	Each local objective $f_{i}$ has $L$-Lipschitz gradient, \emph{\ie},
	$$
	\left\|\nabla f_{i}(x)-\nabla f_{i}(y)\right\| \leq L\|x-y\|
	$$
	for all $i\in\{1,\dots,n\}, x, y \in \mathbb{R}^{d}$, and $f(x^{(0)})-\inf _{x \in \mathbb{R}^{d}} f(x)\leq \Delta $ with $f=\frac{1}{n}\sum_{i=1}^n f_i$.
\end{assumption}

\paragraph{Gradient oracle class. } 
We assume each worker $i$ has  access to its local gradient $\nabla f_i(x)$ via a stochastic gradient oracle $O_i(x;\zeta_i)$ subject to independent randomness $\zeta_i$, \emph{e.g.}, the mini-batch sampling $\zeta_i\triangleq\xi_i\sim D_i$. We further assume that the output $O_i(x,\zeta_i)$ is an  unbiased estimator of the full-batch gradient $\nabla f_i(x)$ with a bounded variance. Formally, we let the {stochastic gradient oracle class $O_{\sigma^2}$} denote the set of all oracles $O_i$ satisfying Assumption \ref{asp:gd-noise}.
\begin{assumption}[\bf Gradient stochasticity]\label{asp:gd-noise}
	For any $x\in\RR^d$, the independent oracles $\{O_i:1\leq i\leq n\}$ satisfy
	\begin{align*}
	\EE_{\zeta_i}[O_i(x;\zeta_i)]=\nabla f_i(x)\quad \text{ and }\quad \EE_{\zeta_i}[\|O_i(x;\zeta_i)-\nabla f_i(x)\|^2]\leq \sigma^2,\quad \forall\, x\in\RR^d\text{ and }i\in\{1,\dots,n\}.
	\end{align*}
\end{assumption}

\paragraph{Compressor class. }
The two  widely-studied classes of compressors  in literature are i) the $\omega$-\emph{unbiased} compressor, described by Assumption \ref{ass:unbiased}, \emph{e.g.}, the stochastic quantization operator  \cite{Alistarh2017QSGDCS}, and ii) the $\delta$-\emph{contractive} compressor, described by Assumption \ref{ass:contract}, \emph{e.g.},
the rand-$k$ \cite{Qian2020ErrorCD} operator and top-$k$ operator \cite{Stich2018SparsifiedSW,Qian2020ErrorCD}. 
\begin{assumption}[\bf Unbiased compressor]\label{ass:unbiased}
	The (possibly random) compression operator
	$C: \RR^d\rightarrow \RR^d$ satisfies
	\begin{equation*}
	\EE[C(x)]=x,\quad \EE[\|C(x)-x\|^2]\leq \omega\|x\|^2,\quad \forall\,x\in\RR^d
	\end{equation*}
	for constant $\omega \ge 0$, where the expectation is taken over the randomness of the compression operator $C$.
\end{assumption}
\begin{assumption}[\bf Contractive compressor]\label{ass:contract}
	The (possibly random) compression operator
	$C: \RR^d\rightarrow \RR^d$ satisfies
	\begin{equation*}
	\EE[\|C(x)-x\|^2]\leq (1-\delta)\|x\|^2,\quad \forall\,x\in\RR^d
	\end{equation*}
	for constant $\delta \in(0,1]$, where the expectation is taken over the randomness of the compression operator $C$.
\end{assumption}
% \textcolor{red}{Joint upper bound: 
% \[
% \EE\left[\left\|\frac{1}{n}\sum_{i=1}^n(C_i(x_i)-x_i)\right\|^2\right]\leq (1-\delta)\left\|\frac{1}{n}\sum_{i=1}^nx_i\right\|^2,\quad \forall\,x_1,\dots,x_n\in\RR^d
% \]
% }
We let $\cU_{\omega}$ and $\cC_\delta$ denote
the set of all $\omega$-unbiased compressors and $\delta$-contractive compressors satisfying Assumptions \ref{ass:unbiased} and \ref{ass:contract}, respectively. Note that the identity operator $I$ satisfies $I \in \cU_\omega$ for all $\omega \ge 0$ and $I \in \cC_\delta$ for all $\delta \in (0,1]$. 
Generally, an $\omega$-unbiased compressor is not necessarily contractive when $\omega$ is larger than $1$. However, since $C\in\cU_\omega$ implies $(1+\omega)^{-1}C\in \cC_{(1+\omega)^{-1}}$, the scaled unbiased compressor is contractive though the converse may not hold. Hence, the class of contractive compressors is strictly richer since it contains all unbiased compressors through scaling.

\paragraph{Algorithm class. } 
We consider a centralized and synchronous algorithm $A$ in which i) workers are allowed to communicate only directly with the central server but not between one another; ii) all iterations are synchronized, meaning that all workers start each of their iterations simultaneously. Each worker $i$ holds a local copy of the model, denoted by $x_i^{(t)}$, at iteration $t$. The output  $\hat{x}^{(t)}$ of $A$ after $t$ iterations can be any linear combination of all previous local models, 
namely, 
\[
\hat{x}^{(t)}\in \mathrm{span}\left(\{x_i^{(s)}:0\leq s\leq t,\,1\leq i\leq n\}\right).
\]
We further require algorithms $A$ to satisfy the so-called ``zero-respecting'' property, which appears in \cite{Carmon2020LowerBF,Carmon2021LowerBF,Lu2021OptimalCI} (see formal definition in Appendix \ref{app:zero-property}). This property implies that the number of non-zero entries of the local model of a worker can be increased only by conducting local stochastic gradient queries or synchronizing with the server. The zero-respecting property holds with all algorithms in Table \ref{tab:introduction-comparison} and most first-order methods based on SGD \cite{Nesterov1983AMF,Kingma2015AdamAM,Huang2021Improved,Zeiler2012ADADELTAAA}. In addition to these properties, algorithm $A$ has to admit communication compression. Specifically, we endow the server with a compressor $C_0$ and each worker $i \in \{1,\cdots, n\}$ with a compressor $C_i$. 
If $C_i = I$ for some $i\in\{0,\cdots, n\}$, then worker $i$ (or the server if $i=0$) conducts lossless communication. 
When $C_i \neq I$ for any $i\in\{0,\cdots, n\}$, algorithm $A$ conducts bidirectional compression. When $C_0 = I$, algorithm $A$ conducts unidirectional compression on messages from workers to server. The definition of the algorithm class with bidirectional/unidirectional compression is as follows. 

\begin{definition}[\bf Algorithm class]\label{def:bidrect-alg}
	Given compressors $\{C_0,C_1,\dots,C_n\}$, write {$\cA^B_{\{C_i\}_{i=0}^n}$} for the set of all centralized, synchronous, zero-respecting algorithms admitting bidirectional compression in which i) compressor $C_i$, $\forall\,i\in\{1,\dots,n\}$, is applied to messages from worker $i$ to the server, and ii) compressor $C_0$ is applied to messages from the server to all workers. 
	
	When $C_0=I$, we write {$\cA^B_{\{ C_0=I\}\cup\{C_i\}_{i=1}^n}$}, or $\cA^U_{\{C_i\}_{i=1}^n}$ for short, for the set of algorithms admitting unidirectional compression. The superscript $B$ or $U$ indicates ``bidirectional'' or ``unidirectional'', respectively.  
\end{definition}

\section{Lower Bounds}\label{sec:lower-bound}
With all interested classes introduced above, we are ready to define the lower bound measure. Given
local loss functions $\{f_i\}_{i=1}^n\subseteq \cF_{\Delta,L}$, stochastic gradient oracles $\{O_i\}_{i=1}^n\subseteq\cO_{\sigma^2}$ (with $O_i$ for worker $i$), compressors $\{\cC_i\}_{i=0}^n \subseteq \cC$ (with $\cC$ being either $\cU_\omega$ or $\cC_\delta$), and an algorithm $A \in \cA$ to solve problem \eqref{eqn:prob} (with $\cA$ being either $\cA^B_{\{C_i\}_{i=0}^n}$ or $\cA^U_{\{C_i\}_{i=1}^n}$), we let $\hat{x}_{{A},\{f_i\}_{i=1}^n,\{O_i\}_{i=1}^n,\{C_i\}_{i=0}^n,T}$ denote the output of algorithm $A$ using no more than $T$ gradient queries and rounds of communication by each worker node. 
We define the minimax measure as 
\begin{align}\label{eqn:measure}
\inf_{A\in \cA} \sup_{\{C_i\}_{i=0}^n\subseteq  \cC}\sup_{\{O_i\}_{i=1}^n\subseteq \cO_{\sigma^2}} \sup_{\{f_i\}_{i=1}^n\subseteq \cF_{\Delta,L}} \EE[\| \nabla f(\hat{x}_{A, \{f_i\}_{i=1}^n,\{O_i\}_{i=1}^n,\{C_i\}_{i=0}^n,T})\|^2].
\end{align}
In \eqref{eqn:measure}, we do not require the compressors $\{C_i\}_{i=0}^n$ to be distinct or independent. When $\cC$ is $\cU_{\omega}$ or $\cC_{\delta}$, we allow the compressor parameter $\omega$ or $\delta$ to be accessible by algorithm $A$. 

% It is worth noting that  the lower bound measure \eqref{eqn:measure} is valid in the minimax sense, which only applies to algorithms admitting general compressors in class $\cC$. Therefore, there might be an algorithm that enjoys a better rate by using specialized compressors.

\subsection{Unidirectional Unbiased Compresssion}\label{sec-lb-uuc}
Our first result is for algorithms that admit unidirectional compression and $\omega$-unbiased compressors.

% \textcolor{red}{For slides purpose}
% \begin{theorem}[\bf Unidirectional unbiased compression]\label{thm:lower1}
% 	There exists a set of  functions $\{f_i\}_{i=1}^n \subseteq \mathcal{F}_{L}$, oracles $\{O_i\}_{i=1}^n\subseteq  \mathcal{O}_{\sigma^{2}}$,  $\omega$-unbiased compressors $\{C_i\}_{i=0}^n\subseteq \cU_\omega$ with $C_0=I$, such that for any algorithm $A \in \mathcal{A}$ starting from a given constant $x^{(0)}$, it holds that 
% 	\begin{equation*}
% 	\EE[\| \nabla f(\hat{x})\|^2]  = \Omega\left(\left(\frac{ L \sigma^2}{nT}\right)^\frac{1}{2}+\frac{(1+\omega)L}{T}\right).
% 	\end{equation*}
% \end{theorem}

\begin{theorem}[\bf Unidirectional unbiased compression]\label{thm:lower1}
	For every $\Delta,\,L>0$, $n\geq 2$, $\omega \geq 0$, $\sigma>0$, $T\geq (1+\omega)^{2}$, there exists a set of local loss functions $\{f_i\}_{i=1}^n \subseteq \mathcal{F}_{\Delta, L}$, stochastic gradient oracles $\{O_i\}_{i=1}^n\subseteq  \mathcal{O}_{\sigma^{2}}$,  $\omega$-unbiased compressors $\{C_i\}_{i=0}^n\subseteq \cU_\omega$ with $C_0=I$, such that for any algorithm $A \in \mathcal{A}^U_{\{C_i\}_{i=1}^n}$ starting from a given constant $x^{(0)}$, it holds that  (proof is in Appendix \ref{app:proof-lb-1})
	\begin{equation}\label{eqn:lower-bound1}
	\EE[\| \nabla f(\hat{x}_{A, \{f_i\}_{i=1}^n,\{O_i\}_{i=1}^n,\{C_i\}_{i=0}^n,T})\|^2]  = \Omega\left(\left(\frac{\Delta L \sigma^2}{nT}\right)^\frac{1}{2}+\frac{(1+\omega)\Delta L}{T}\right).
	\end{equation}
	% In particular, if  $T_{\mathrm{q}}/T_{\mathrm{c}}=\Theta(1)$, \ie,  $T_{\mathrm{q}}=\Theta(T)=\Tc$ for some $T$, then this lower bound becomes 
	% \begin{align}\label{eqn:lower-bound1-2}
	%      \Omega\left(\left(\frac{\Delta L \sigma^2}{nT}\right)^\frac{1}{2}+\frac{(1+\omega)\Delta L}{T}\right).
	% \end{align}
\end{theorem}

\paragraph{Consistency with previous works. } 
The bound in \eqref{eqn:lower-bound1} is consistent with best-known lower bounds in different settings. When $\omega = 0$, our result reduces to the tight bound for distributed training without compression \cite{Arjevani2019LowerBF}. When $n=1$ and $\omega=0$, our result reduces to the lower bound established in \cite{Arjevani2019LowerBF} under the single-node non-convex stochastic setting. When $n=1, \omega=0$ and $\sigma^2=0$, our result recovers the tight bound
for deterministic non-convex optimization \cite{Carmon2020LowerBF}. 

\paragraph{Linear-speedup. } 
	When $T$ is sufficiently large, the first term $1/\sqrt{nT}$ dominates the lower bound \eqref{eqn:lower-bound1}. If an algorithm achieves an $O(1/\sqrt{nT})$ rate, it will require $T = O(1/(n\epsilon^2))$ gradient queries to reach a desired accuracy $\epsilon$, which is inversely proportional to $n$. Therefore, an algorithm achieves linear-speedup at $T$-th iteration if, for this $T$, the term involving $nT$ is dominating the rate.

\paragraph{Transient complexity. } 
	Due to the compression-incurred overhead in convergence rate, a distributed stochastic algorithm with communication compression has to experience a transient stage to achieve its linear-speedup stage. Transient complexity are referred to the number of gradient queries (or communication rounds) when $T$ is relatively small so non-$nT$ terms still dominate the rate. The smaller the transient complexity is, the less gradient queries or communication rounds  the algorithm requires to achieve  linear-speedup stage. For example, if an algorithm can achieve the lower bound established in \eqref{eqn:lower-bound1}, it requires 
	$(\frac{\Delta L \sigma^2}{nT})^\frac{1}{2} \ge \frac{(1+\omega)\Delta L}{T}$, \emph{\ie}, $ T = \Omega(n(1+\omega)^2)$
% 	\begin{align}
% 		\left(\frac{\Delta L \sigma^2}{nT}\right)^\frac{1}{2} \ge \frac{(1+\omega)\Delta L}{T}\ \quad \Longrightarrow \quad T = O(n(1+\omega)^2)
% 	\end{align}
	transient gradient queries (or communication rounds) to achieve linear-speedup, which is proportional to the compression-related terms $(1+\omega)^2$. Here we mainly care about the orders with respect to $n$ and $\omega$ (or $\delta$ below) in transient complexity to evaluate how sensitive the algorithm is to compression. Transient complexity is  also widely used in decentralized learning \cite{pu2019sharp,ying2021exponential} to gauge how network topology can influence the convergence rate.

\subsection{Bidirectional Unbiased Compression}
Theorem \ref{thm:lower1} applies to unidirectional compression where $C_0 = I$. We next consider bidirectional compression with $C_0\in \cU_\omega$. Since $\{C_i\}_{i=1}^n\subseteq \cU_\omega$ with $C_0=I$ is a special case of $\{C_i\}_{i=0}^n\subseteq \cU_\omega$, the lower bound for algorithms that admit bidirectional compression is greater than or equal to that with unidirectional compression by following the definition of the measure in \eqref{eqn:measure}. 

% \textcolor{red}{For slides purpose}
% \begin{corollary}[\bf Bidirectional unbiased compression]\label{cor:lower2}
% 	The same lower bound holds for any algorithms conducting bidirectional unbiased compression.
% \end{corollary}

\begin{corollary}[\bf Bidirectional unbiased compression]\label{cor:lower2}
	Under the same setting as in Theorem \ref{thm:lower1},  there exists a set of local objectives $\{f_i\}_{i=1}^n \subseteq \mathcal{F}_{\Delta, L}$, stochastic gradient oracles $\{O_i\}_{i=1}^n\subseteq  \mathcal{O}_{\sigma^{2}}$,  $\omega$-unbiased compressors $\{C_i\}_{i=0}^n\subseteq \cU_\omega$ such that for any algorithm $A \in \mathcal{A}^B_{\{C_i\}_{i=0}^n}$  starting from $x^{(0)}$, the lower bound in  \eqref{eqn:lower-bound1} is also valid.
	% In particular, if   $T_{\mathrm{q}}=\Theta(T)=\Tc$ for some $T$, then the lower bound \eqref{eqn:lower-bound1-2} is also valid.
\end{corollary}
Theorem \ref{thm:lower1} and Corollary \ref{cor:lower2} indicate that distributed learning with both unidirectional and bidirectional communication compression share the same lower bound. It is intuitive since unidirectional compression is just a special case of bidirectional compression by letting $C_0 = I$.

\subsection{Unidirectional Contractive Compression}
To obtain lower bounds for contractive compressors, we need the following lemma \cite[Lemma 1]{Safaryan2020UncertaintyPF}.

% \textcolor{red}{For slides purpose}
% \begin{lemma}[\bf Compressor relation]\label{lem:un-con}
% 	It holds that $(1+\omega)^{-1}\,\cU_{\omega} \triangleq\{(1+\omega)^{-1}\,C:C\in \cU_{\omega}\}\subseteq \cC_{(1+\omega)^{-1}}$.
% \end{lemma}

\begin{lemma}[\bf Compressor relation]\label{lem:un-con}
	It holds that $\delta\,\cU_{\delta^{-1}-1} \triangleq\{\delta \,C:C\in \cU_{\delta^{-1}-1}\}\subseteq \cC_{\delta}$.
\end{lemma}
The above Lemma reveals that any $(\delta^{-1}-1)$-unbiased compressor is $\delta$-contractive when scaled by $\delta$. Therefore, if an algorithm $A$ admits all $\delta$-contractive compressors, it automatically admits all compressors in $\delta\,\cU_{\delta^{-1}-1}$ due to Lemma \ref{lem:un-con}. This relation, together with Theorem \ref{thm:lower1}, helps us achieve the following lower bound with respect to $\delta$-contractive compressors.

% \textcolor{red}{For slides purpose}
% \begin{theorem}[\bf Unidirectional unbiased compression]\label{thm:lower1}
% 	There exists a set of  functions $\{f_i\}_{i=1}^n \subseteq \mathcal{F}_{L}$, oracles $\{O_i\}_{i=1}^n\subseteq  \mathcal{O}_{\sigma^{2}}$,  $\delta$-contractive compressors $\{C_i\}_{i=0}^n\subseteq \cC_\delta$ with $C_0=I$, such that for any algorithm $A \in \mathcal{A}$ starting from a given constant $x^{(0)}$, it holds that 
% 	\begin{equation*}
% 	\EE[\| \nabla f(\hat{x})\|^2]  = \Omega\left(\left(\frac{ L \sigma^2}{nT}\right)^\frac{1}{2}+\frac{L}{\delta T}\right).
% 	\end{equation*}
% \end{theorem}

\begin{theorem}[\bf Unidirectional contractive compression]\label{thm:lower3}
	For every $\Delta,\,L>0$, $n\geq 2$, $0<\delta\leq 1$, $\sigma>0$, $T\geq\delta^{-2}$, there exists a set of loss objectives $\{f_i\}_{i=1}^n \subseteq \mathcal{F}_{\Delta, L}$, a set of stochastic gradient oracles $\{O_i\}_{i=1}^n\subseteq  \mathcal{O}_{\sigma^{2}}$, a set of $\delta$-contractive compressors $\{C_i\}_{i=0}^n\subseteq \cC_\delta$ with $C_0=I$, such that for any algorithm $A \in \mathcal{A}^U_{\{C_i\}_{i=0}^n}$ starting from $x^{(0)}$, it holds that (proof is in Appendix \ref{app:proof-lb-2})
	\begin{equation}\label{eqn:lower-bound22}
	\EE[\|  \nabla f(\hat{x}_{A, \{f_i\}_{i=1}^n,\{O_i\}_{i=1}^n,\{C_i\}_{i=0}^n,T})\|^2] = \Omega\left(\left(\frac{\Delta L \sigma^2}{nT}\right)^\frac{1}{2}+\frac{\Delta L}{\delta T}\right).
	\end{equation}
\end{theorem}
\textbf{Transient complexity.} 
	With the discussion on transient complexity in Section~\ref{sec-lb-uuc}, it is easy to derive the transient iteration complexity as $\Omega(n/\delta^2)$ for the lower bound with $\delta$-contractive compressors.

\subsection{Bidirectional Contractive Compression}
Noting that $\{C_i\}_{i=1}^n\subseteq \cC_\delta$ with $C_0=I$ is a special case of $\{C_i\}_{i=0}^n\subseteq \cC_\delta$, we can also establish the lower bound for algorithms that admit bidirectional compression and contractive compressors.

% \textcolor{red}{For slides purpose}
% \begin{corollary}[\bf Bidirectional contractive compression]\label{cor:lower2}
% 	The same lower bound holds for any algorithms conducting bidirectional contractive compression.
% \end{corollary}

\begin{corollary}[\bf Bidirectional contractive compression]\label{cor:lower4}
	Under the same settings as in Theorem \ref{thm:lower3}, there exists a set of loss objectives $\{f_i\}_{i=1}^n \subseteq \mathcal{F}_{\Delta, L}$, a set of stochastic gradient oracles $\{O_i\}_{i=1}^n\subseteq  \mathcal{O}_{\sigma^{2}}$, a set of $\delta$-contractive compressors $\{C_i\}_{i=1}^n\subseteq \cC_\delta$, such that for any algorithm $A \in \mathcal{A}^B_{\{C_i\}_{i=0}^n}$ starting form $x^{(0)}$, the lower bound \eqref{eqn:lower-bound22} is also valid.
\end{corollary}

\section{NEOLITHIC: A Nearly Optimal Algorithm}\label{sec:algorithm}
\vspace{-1mm}
Comparing the best-known upper bounds listed in Table \ref{tab:introduction-comparison} with the established lower bounds in \eqref{eqn:lower-bound1} and \eqref{eqn:lower-bound22}, we find existing algorithms may  be sub-optimal. There exists a clear gap between their convergence rates and our established lower bounds. In this section, we propose NEOLITHIC to fill in this gap. Its rate will match the lower bounds established in \eqref{eqn:lower-bound1} and \eqref{eqn:lower-bound22} up to logarithm factors. NEOLITHIC can work with both unidirectional and bidirectional compressions, and it is compatible with both unbiased and contractive compressors. NEOLITHIC will be discussed in detail with bidirectional contractive compression in this section.  It is easy to be adapted to other settings by simply removing the compression on the server side or utilizing unbiased compressors with proper scaling.

% \textcolor{red}{For slides purpose}
% \begin{algorithm}[t]
% 	\caption{Fast Compressed Communication: $v^{(k,R)}=\text{FCC}(v^{(k,\star)},C, R,\text{receiver(s)})$}
% 	\label{alg:fcc}
% 	\begin{algorithmic}
% 		\STATE \noindent {\bfseries Input:} The  vector $v^{(k,\star)}$ aimed to communicate at iteration $k$; a compressor $C$; 
% 		\STATE $\hspace{1.2cm}$ rounds $R$; initial vector $v^{(k,0)}=0$; receiver(s); 
% 		\FOR{$r=0,\cdots,R-1$}
% 		\STATE Compress $v^{(k,\star)}-v^{(k,r)}$ into $c^{(k,r)}=C(v^{(k,\star)}-v^{(k,r)})$ 
% 		\STATE Send $c^{(k,r)}$ to the target receiver(s)
% 		\STATE Update $v^{(k,r+1)}=v^{(k,r)}+c^{(k,r)}$
% 		\ENDFOR {\color{blue}$\hspace{2.05cm} \triangleright$ {\footnotesize $\{c^{(k,r)}\}_{r=0}^{R-1}$ will be sent to the receiver during the for-loop}}
% 		\RETURN  $v^{(k,R)}$. {\color{blue}$\hspace{0.9 cm} \triangleright$ {\footnotesize Fact: $v^{(k,R)}=\sum_{r=0}^{R-1}c^{(k,r)}$}}
% 	\end{algorithmic}
% \end{algorithm}

\begin{algorithm}[t]
	\caption{Fast Compressed Communication: $v^{(k,R)}=\text{FCC}(v^{(k,\star)},C, R,\text{target receiver(s)})$}
	\label{alg:fcc}
	\begin{algorithmic}
		\STATE \noindent {\bfseries Input:} The  vector $v^{(k,\star)}$ aimed to communicate at iteration $k$; a compressor $C$; 
		\STATE $\hspace{1cm}$ rounds $R$; initial vector $v^{(k,0)}=0$; target receiver(s); 
		\FOR{$r=0,\cdots,R-1$}
		\STATE Compress $v^{(k,\star)}-v^{(k,r)}$ into $c^{(k,r)}=C(v^{(k,\star)}-v^{(k,r)})$ 
		\STATE Send $c^{(k,r)}$ to the target receiver(s)
		\STATE Update $v^{(k,r+1)}=v^{(k,r)}+c^{(k,r)}$
		\ENDFOR {\color{blue}$\hspace{3.725cm} \triangleright$ {\footnotesize The set $\{c^{(k,r)}\}_{r=0}^{R-1}$ will be sent to the receiver during the for-loop}}
		\RETURN Variable $v^{(k,R)}$. {\color{blue}$\hspace{1.1 cm} \triangleright$ {\footnotesize It holds that $v^{(k,R)}=\sum_{r=0}^{R-1}c^{(k,r)}$}}
	\end{algorithmic}
\end{algorithm}

\subsection{Fast Compressed Communication}
NEOLITHIC is built on a communication compression  module listed in Algorithm \ref{alg:fcc}, which we call fast compressed communication (FCC). Given an input vector $v^{(k,\star)}$ to communicate in the $k$-th iteration, FCC first initializes $v^{(k,0)}=0$ and then recursively compresses the residual with $c^{(k,r)}\triangleq C(v^{(k,\star)}-v^{(k,r)})$ and sends it to the receiver for $R$ consecutive rounds, see the main recursion in Algorithm \ref{alg:fcc}. When FCC module ends, the sender will transmit a set of compressed variables $\{c^{(k,r)}\}_{r=0}^{R-1}$ to the receiver, and return $v^{(k,R)}=\sum_{r=0}^{R-1}c^{(k,r)}$ to itself. Quantity $v^{(k,R)}$ can be regarded as a compressed vector of original input $v^{(k,\star)}$ after FCC operation.

FCC module can be conducted by the server (for which all variables in FCC are without any subscripts, {\em e.g.,} $c^{(k,r)}$ and $v^{(k,r)}$) or by any worker $i$ (for which all variables are with a subscript $i$, {\em e.g.,} $c_i^{(k,r)}$ and $v_i^{(k,r)}$). When $R=1$, FCC reduces to the standard compression utilized in existing literature \cite{Tang2019DoubleSqueezePS,Stich2018SparsifiedSW,Fatkhullin2021EF21WB,Xie2020CSERCS,Jiang2018ALS}. While FCC requires $R$ rounds of communication per iteration, the following lemma establishes that the compression error decrease exponentially with respect to communication rounds $R$. 

% \textcolor{red}{for slides purpose}
% \begin{lemma}[\bf FCC property]
% 	\label{lem:fcc-expo}
% 	Let $C$ be a $\delta$-contractive compressor and  $v^{(k,R)} =\mathrm{FCC}(v^{(k,\star)},C,R)$. Then,
% 	\begin{equation*}\label{znads}
% 	\EE[\|v^{(k,R)}-v^{(k,\star)}\|^2]\leq (1-\delta)^R\|v^{(k,\star)}\|^2, \quad \forall k=0,1,2,\cdots.
% 	\end{equation*}
% \end{lemma}

\begin{lemma}[\bf FCC property]
	\label{lem:fcc-expo}
	Let $C$ be a $\delta$-contractive compressor and  $v^{(k,R)} =\mathrm{FCC}(v^{(k,\star)},C,R)$. It holds for any $R\ge 1$ and $v^{(k,\star)} \in \RR^d$  that (proof is in Appendix \ref{app:fcc})
	\begin{equation}\label{znads}
	\EE[\|v^{(k,R)}-v^{(k,\star)}\|^2]\leq (1-\delta)^R\|v^{(k,\star)}\|^2, \quad \forall k=0,1,2,\cdots.
	\end{equation}
\end{lemma}

When $R=1$, the above FCC property \eqref{znads} reduces to Assumption \ref{ass:contract} for standard contractive compressors.
When $R$ is large, FCC can output $v^{(k,R)}$ endowed with very small compression errors. {The FCC protocol is also closely related to EF21 compression strategy \cite{Richtrik2021EF21AN}\footnote{If the roles of $v$ and $v^\star$ in Eq. (8) of \cite{Richtrik2021EF21AN} are switched, we can get the single round of FCC recursions.}. However, the major novelty of FCC is to utilize multiple such compression rounds to help develop algorithms that can nearly match the established lower bounds.} 

\subsection{NEOLITHIC Algorithm} 
NEOLITHIC is described in Algorithm \ref{alg:mcdc}. The FCC module in NEOLITHIC communicates $R$ rounds per iteration. To balance  gradient queries and communication rounds, NEOLITHIC will query $R$ stochastic gradients per iteration, see the gradient accumulation step in Algorithm \ref{alg:mcdc}. Compared to other algorithms listed in Table \ref{tab:introduction-comparison}, the proposed NEOLITHIC takes $R$ times more  gradient queries and communication rounds than them per iteration. Given the same budgets to query gradient oracles and conduct communication as the other algorithms, say $T$ times on each worker, we shall consider $K = T/R$ iterations in NEOLITHIC for fair comparison.

\begin{algorithm}[t]
	\caption{NEOLITHIC}
	\label{alg:mcdc}
	\begin{algorithmic}
		\STATE \noindent {\bfseries Input:} Initialize $x^{(0)}$; learning rate $\gamma$; compression round $R$; $\er^{(0)}=\er_i^{(0)}=0$, $\forall\, i\in[n]$
		\FOR{$k=0,1,\cdots, K-1$}
		\STATE  {\bfseries On all workers in parallel:}
		\STATE \quad Query stochastic gradients $\hat{g}_i^{(k)}=\frac{1}{R}\sum_{r=0}^{R-1}O_i(x^{(k)};\zeta_i^{(k,r)})$ {\color{blue}\hspace{1.2cm}$\triangleright $ {\footnotesize Gradient accumulation}}
		\STATE \quad Error compensate $\tilde{g}_i^{(k)}=\hat{g}_i^{(k)}+\er_i^{(k)}$
		\STATE \quad Update error $\er_i^{(k+1)}=\tilde{g}_i^{(k)}-\text{FCC}(\tilde{g}_i^{(k)},C_i,R, \text{server})$
		{\color{blue}\hspace{1.82cm}$\triangleright $ \footnotesize Worker sends $\{c_i^{(k,r)}\}$ to server}
		\STATE  {\bfseries On server:}
		\STATE \quad Error compensate  $\tilde{g}^{(k)}=\frac{1}{n}\sum_{i=1}^n\sum_{r=0}^{R-1}c_i^{(k,r)}+\er^{(k)}$ {\color{blue}\hspace{1.9cm} $\triangleright $ \footnotesize $\{c_i^{(k,r)}\}$ received from workers}
		\STATE \quad Update error ${\er}^{(k+1)}=\tilde{g}^{(k)}-\text{FCC}(\tilde{g}^{(k)},C_0,R, \text{all workers})$ {\color{blue}\hspace{1.02cm}$ \triangleright $ \footnotesize Server sends $\{c^{(k,r)}\}$ to workers} 
		\STATE  {\bfseries On all workers in parallel:}
		\STATE \quad Update model parameter $x^{(k+1)}=x^{(k)}-\gamma \sum_{r=0}^{R-1}c^{(k,r)}$ {\color{blue}\hspace{1.55cm}$ \triangleright $ \footnotesize $\{c^{(k,r)}\}$ received from server}
		\ENDFOR
	\end{algorithmic}
\end{algorithm}

We introduce the following assumption to establish the convergence rates of NEOLITHIC.
\begin{assumption}[\bf Gradient dissimilarity]\label{asp:gd-hetero}
There exists some  $b^2 \ge 0$ such that
\begin{equation*}
    \frac{1}{n}\sum_{i=1}^n\|\nabla f_i(x)-\nabla f(x)\|^2\leq b^2,\quad \forall\,x\in\RR^d.
\end{equation*}
\vspace{-4mm}
\end{assumption}
When local distributions $D_i$ are identical across all workers, we have $f_i(x) = f(x)$ for each $i\in\{1,\dots,n\}$ and the above assumption will always hold. 
We first establish the convergence rate of NEOLITHIC using  bidirectional compression with contractive compressors. 

% \newpage
% \textcolor{red}{for slides purpose}
% \begin{theorem}[\bf NEOLITHIC with bidirectional contractive compression]
% \label{thm:convergence-rate}
% By setting $R$ and the learning rate appropriately,  it holds for any $K\geq 0$ and compressors $\{C_i\}_{i=0}^n \subseteq \cC_\delta$ that
% \begin{equation*}
%     \frac{1}{K+1}\sum_{k=0}^K \EE[\|\nabla f(x^{(k)})\|^2] =  \tilde{O}\left( \left(\frac{\Delta L \sigma^2}{nT}\right)^\frac{1}{2}+\frac{\Delta L}{\delta T}\right),
% \end{equation*}
% where $T=KR$ is the total number of gradient queries/communication rounds on each worker.  
% \end{theorem}

\begin{theorem}[\bf NEOLITHIC with bidirectional contractive compression]
\label{thm:convergence-rate}
Given constants $n\ge 1$, $\delta \in (0,1]$ and Assumption \ref{asp:gd-hetero}, and let $x^{(k)}$ be generated by Algorithm \ref{alg:mcdc}. If  $R=\lceil \frac{\max\{\ln\left({\delta T\max\{b^2,\sigma^2\delta \}}/{\Delta L}\right),\ln(8)\}}{\delta}\rceil$ and the learning rate is set as in Appendix \ref{app:convergence-rate},  it holds for any $K\geq 0$ and compressors $\{C_i\}_{i=0}^n \subseteq \cC_\delta$ that (proof is in Appendix \ref{app:convergence-rate})
\begin{equation*}
    \frac{1}{K+1}\sum_{k=0}^K \EE[\|\nabla f(x^{(k)})\|^2] =  \tilde{O}\left( \left(\frac{\Delta L \sigma^2}{nT}\right)^\frac{1}{2}+\frac{\Delta L}{\delta T}\right),
\end{equation*}
where $T=KR$ is the total number of gradient queries/communication rounds on each worker and notation $\tilde{O}(\cdot)$ hides logarithmic factors. The above rate implies a transient complexity of $\tilde{O}(n\delta^{-2})$.  
\end{theorem}

When  $\{C_i\}_{i=0}^n$ are $\omega$-unbiased compressors, we utilize the fact that $(1+\omega)^{-1}C_i$ is $(1+\omega)^{-1}$-contractive for all $i\in\{0,\dots,n\}$ to derive that

\begin{corollary}[\bf NEOLITHIC with bidirectional unbiased compression]
\label{cor:convergence-rate-unbiased}
Under the same assumptions as in Theorem \ref{thm:convergence-rate}, it holds for any $K\geq 0$ and compressors $\{C_i\}_{i=0}^n \subseteq \cU_\omega$ that
\begin{equation*}
    \frac{1}{K+1}\sum_{k=0}^K \EE[\|\nabla f(x^{(k)})\|^2] =  \tilde{O}\left( \left(\frac{\Delta L \sigma^2}{nT}\right)^\frac{1}{2}+\frac{(1+\omega)\Delta L}{T}\right).
\end{equation*}
This further leads to a transient complexity of $\tilde{O}\left(n(1+\omega)^2\right)$.
\end{corollary}

\begin{remark}\label{rmk-unbiased-unidirectional}
The convergence rates established in Theorem \ref{thm:convergence-rate} and Corollary \ref{cor:convergence-rate-unbiased} are also valid for unidirectional compression when $C_0 = I$. They can match the lower bounds established in Section~\ref{sec:lower-bound} up to logarithm factors. Moreover, these rates are faster than other algorithms listed in Table \ref{tab:introduction-comparison}. 
\end{remark}
\begin{remark} \label{rmk:no-benefit}
While our analysis relies on Assumption  \ref{asp:gd-hetero}, the rates in Theorem \ref{thm:convergence-rate} and Corollary \ref{cor:convergence-rate-unbiased} are polynomially independent of $b^2$.
These results 
also imply that NEOLITHIC with bidirectional compression can perform as fast as its counterpart with unidirectional compression. In other words, imposing bidirectional compression can save communications in NEOLITHIC without hurting convergence rates. Before our work, it is established in literature \cite{Stich2018SparsifiedSW,Sahu2021RethinkingGS,Tang2019DoubleSqueezePS,Xu2021StepAheadEF,Beznosikov2020OnBC} that bidirectional compression leads to slower convergence than unidirectional compression. Our results also imply that, given $\delta=(1+\omega)^{-1}$, 
well-designed algorithms such as NEOLITHIC with $\delta$-contractive compressors can converge as fast as their counterparts using $\omega$-unbiased compressors despite unbiasedness. Before our work, the analysis based on unbiased compressors \cite{Beznosikov2020OnBC, Mishchenko2019DistributedLW, Horvath2019StochasticDL,Gorbunov2021MARINAFN} exhibits theoretically  faster convergence for non-convex problems.
\end{remark}

\begin{remark}
We remark that Assumption \ref{asp:gd-hetero} is not required to obtain the lower bounds in Section~\ref{sec:lower-bound}. It is not known whether the lower bounds established in Section~\ref{sec:lower-bound} can be achieved by NEOLITHIC when Assumption \ref{asp:gd-hetero} does not hold. However, it is worth noting that Assumption \ref{asp:gd-hetero} is already milder than those made in most works  \cite{Zhao2019GlobalMC,Karimireddy2019ErrorFF,Stich2018SparsifiedSW,Xie2020CSERCS,Tang2019DoubleSqueezePS,Basu2020QsparseLocalSGDDS} such as  bounded gradients.  
To our best knowledge, only EF21-SGD \cite{Fatkhullin2021EF21WB} is guaranteed to converge without Assumption \ref{asp:gd-hetero}, which, however, leads to a fairly loose convergence rate (see Table \ref{tab:introduction-comparison}) that  cannot show linear-speedup $O(1/\sqrt{nT})$.
\end{remark}

\begin{remark}
In the deterministic scenario, \ie, $\sigma^2=0$, NEOLITHIC produces rate $\tilde{O}((1+\omega)\Delta L/T)$. There exist literature (\eg, \cite{Gorbunov2021MARINAFN,Tyurin2022DASHADN}) that can achieve rate $\tilde{O}((1+\omega/n)\Delta L/T)$ which outperforms NEOLITHIC especially when $\omega$ and $n$ are sufficiently large. However, this improvement is built upon an additional assumption that all local compressors $\{C_i\}_{i=0}^n$ are independent of each other and cannot share the same randomness. On the contrary, NEOLITHIC does not impose such an assumption and can be applied to both dependent and independent compressors.
\end{remark}

\section{Experiments}
This section empirically investigates the performance of different compression algorithms with both synthetic simulation and deep learning tasks. We compare \ours with \psgd and its  variants with communication compression: \memsgd, \ds, and \ef. 

\subsection{Synthetic Datasets}
\paragraph{Least square.}
We consider the following least-square problem:
\begin{equation*}
    \min_{x\in\RR^d}\quad \frac{1}{2n}\sum_{i=1}^n \|A_ix-b_i\|^2,
\end{equation*}
where coefficient matrix $A_i\in \RR^{M\times d}$ and measurement $b_i\in\RR^d$ are associated with node $i$, and $M$ is the size of local dataset. We set $d = 30$, $n=32$ and $M = 1000$, and generate data by letting each node $i$ be associated with a local solution $x_i^\star$ randomly generated by $\cN(0,I_d)$. Then we generate each element in $A_i$ following standard normal distribution, and  measurement $b_i$ is generated by $b_i=A_ix_i^\star +s_i$ with white noise $s_i\sim \cN(0, 0.01)$. At each query, every node will randomly sample a row in $A_i$ and the corresponding element in $b_i$ to evaluate the stochastic gradient. 
We adopt the rand-1 compressor and set the number of rounds $R=4$ for NEOLITHIC. 
e use stair-wise decaying learning rates in which the learning
rates are divided by every $2,500$ communication rounds. Each algorithm is averaged with $20$ trials
The result is shown in Figure~\ref{fig:synthetic} (left). It is observed that NEOLITHIC outperforms MEM-SGD and Double-Squeeze in convergence rate, and it performs closely to P-SGD.

\begin{figure}[t!]
	\centering
		\includegraphics[width=0.4\textwidth]{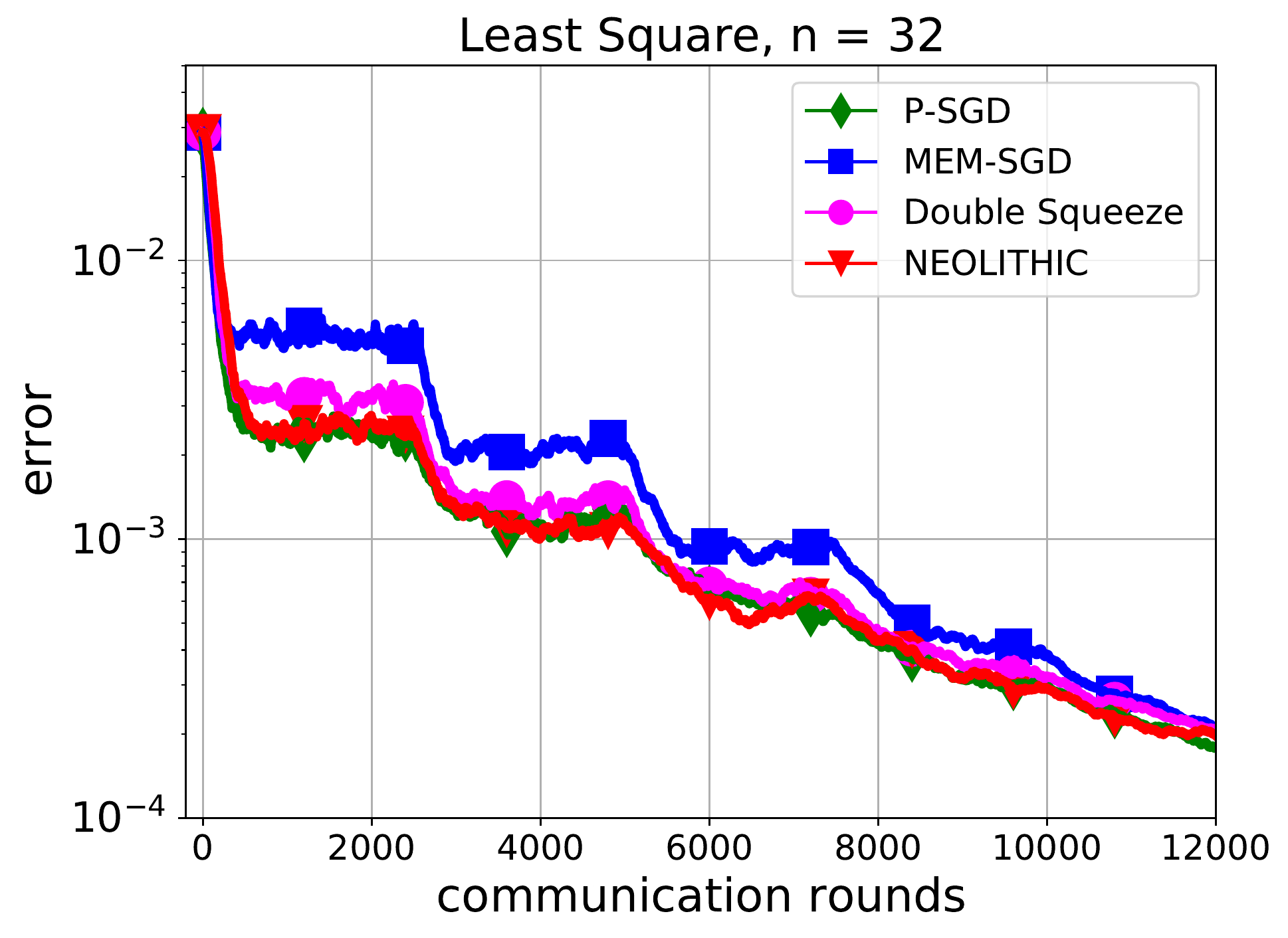}
		\includegraphics[width=0.4\textwidth]{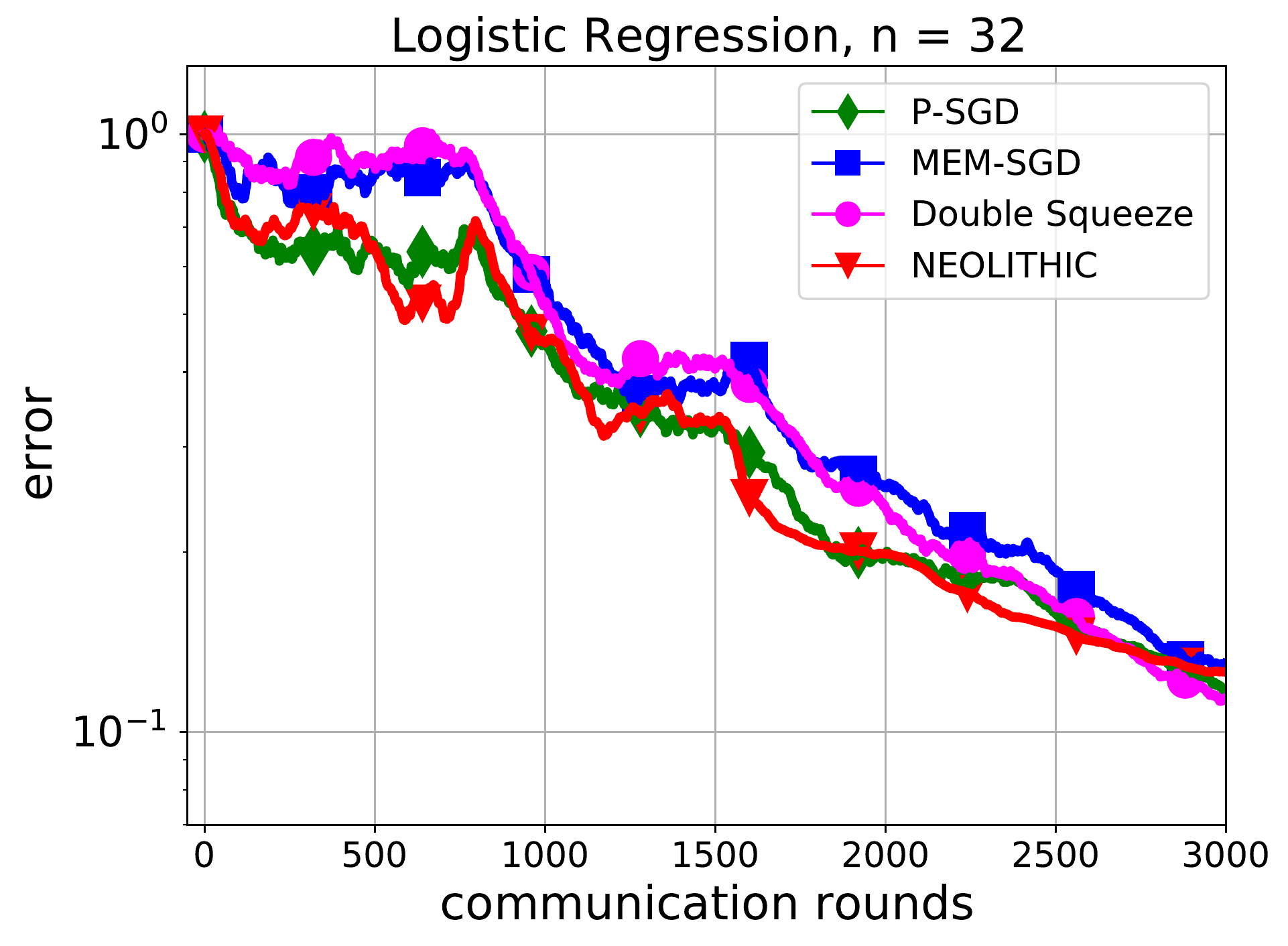}
\caption{\small Convergence results on synthetic problems in terms of  $\EE[\|\nabla f(x)\|^2]$ versus communication rounds.}
	\label{fig:synthetic}
\end{figure}

\paragraph{Logistic regression. }
We consider the following logistic regression problem:
\[
\min _{x \in \mathbb{R}^{d}}\quad  \frac{1}{n} \sum_{i=1}^{n} f_{i}(x) \quad\text{ where }\quad f_{i}(x)=\frac{1}{M} \sum_{m=1}^{M} \ln \left(1+\exp \left(-y_{i, m} h_{i, m}^{\top} x\right)\right)
\]
where $\{h_{i,m},y_{i,m}\}_{m=1}^M$ is the training dateset held by node $i$ in which $h_{i,m}\in\RR^d$ with $d=30$ is a feature vector while $y_{i,m}\in\{-1,+1\}$ is the corresponding label. Similar to the least square problem, each node $i$ is associated with a local solution $x_i^\star$. We generate each feature vector $h_{i,m}\sim \cN(0,I_d)$, and label $y_{i,m} = 1$ with probability $1/(1+\exp(-y_{i,m}h_{i,m}^\top x_i^\star))$; otherwise $y_{i,m}=-1$.
At each query, every node will randomly sample a datapoint $(h_{i,m},y_{i,m})$  to evaluate the stochastic gradient.
We adopt the rand-1 compressor and set the number of rounds $R=4$ for NEOLITHIC.
We use stair-wise decaying
learning rates in which the learning rates are divided by every $800$ communication rounds. Each
algorithm is averaged with $20$ trials.
The result is shown in Figure~\ref{fig:synthetic} (right). Again, NEOLITHIC outperforms MEM-SGD and Double-Squeeze in convergence rate, and it performs closely to P-SGD.

\subsection{Deep Learning Tasks}
\paragraph{Implementation details. }We implement all compression algorithms with PyTorch \cite{paszke2019pytorch} 1.8.2 using NCCL 2.8.3 (CUDA 10.1) as the communication backend. For P-SGD, we used PyTorch’s native Distributed Data Parallel (DDP) module. All deep learning training scripts in this section run on a server with 8 NVIDIA V100 GPUs in our cluster and each GPU is treated as one worker.

\begin{figure}[t!]
	\centering
		\includegraphics[width=0.4\linewidth]{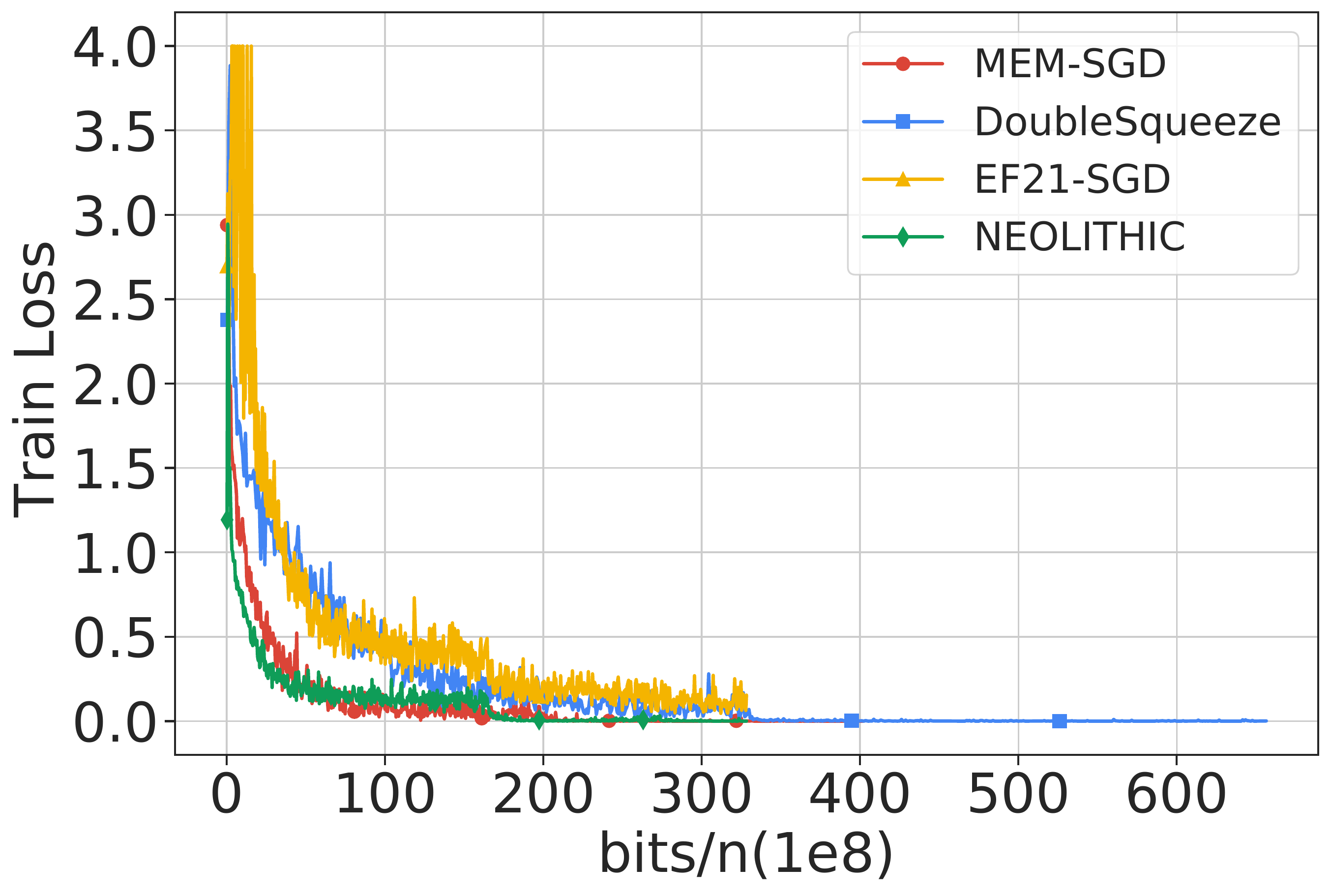}
		\includegraphics[width=0.4\linewidth]{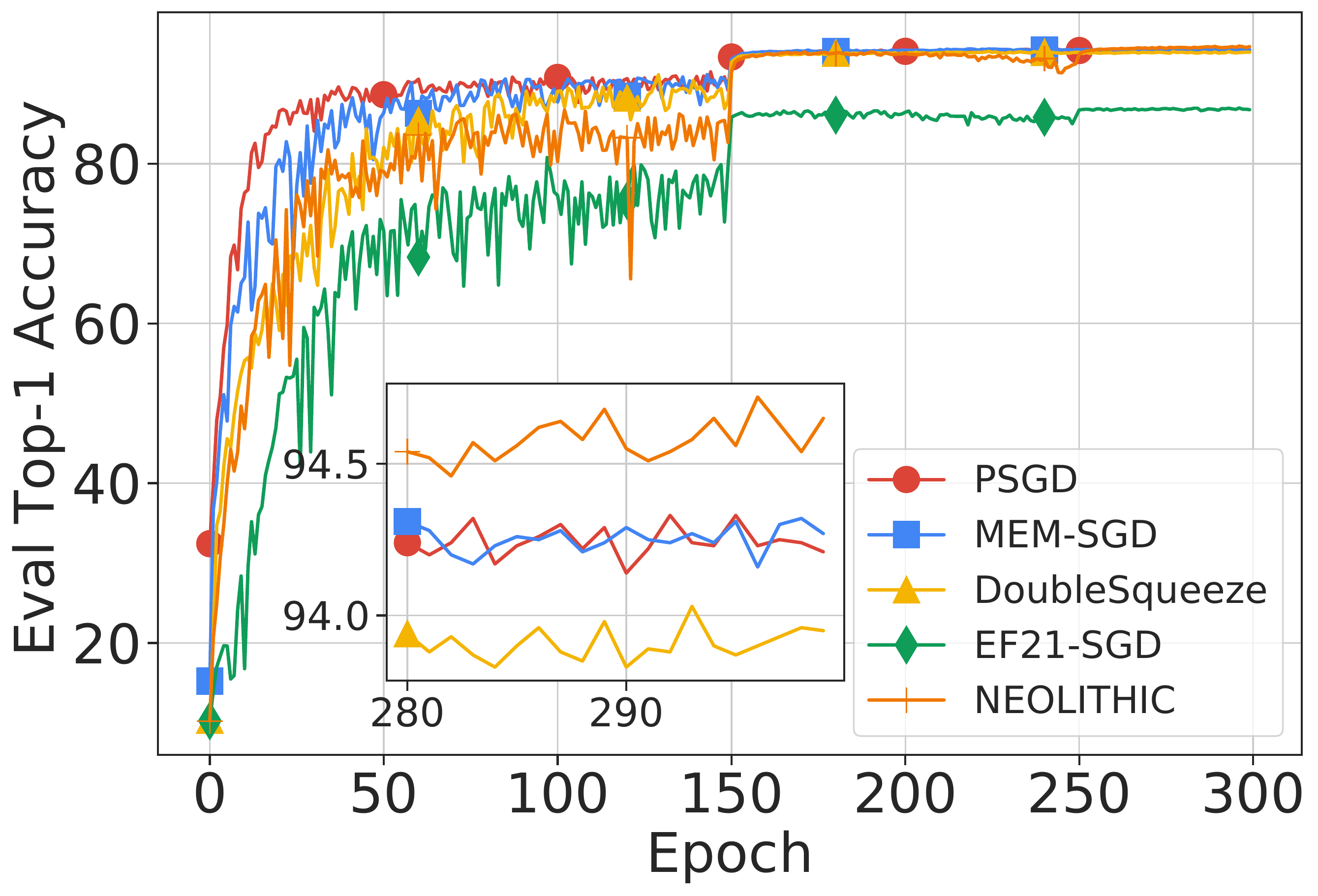}
\caption{\small Convergence results on the CIFAR-10 in terms of training loss and validation accuracy.}
	\label{fig:cifar10}
\end{figure}

\paragraph{Image classification. }We investigate the performance of the aforementioned methods with \cifar\cite{krizhevsky2009learning} dataset. For \cifar dataset, it consists of 50,000 training images and 10,000 validation images categorized in 10 classes. 
We utilize two common variants of ResNet \cite{he2016deep} model on \cifar (ResNet-20 with roughly 0.27M parameters and ResNet-18 with 11.17M parameters). We train total 300 epochs and set the batch size to 128 on every worker. The learning rate is set to 5e-3 for single worker and warmed up in the first 5 epochs and decayed by a factor of 10 at 150 and 250-th epoch.

% \textcolor{red}{for slides purpose}
% \begin{table}
% \tablefontsize
% \caption{\small Accuracy comparison with different algorithms on CIFAR-10.}

% \begin{center}
% \begin{small}
% \begin{sc}

% \begin{tabular}{ccccc}
% \toprule

% Comp. ratio &Methods   & ResNet18     & ResNet20   \\

% \midrule

% $-$&PSGD & 93.99 ± 0.52 & 91.62 ± 0.13  \\ 
% \midrule
% \multirow{4}{*}{5\%}&MEM-SGD & 94.35 ± 0.01 & 91.27 ± 0.08 \\
% &Double-Squeeze & 94.11 ± 0.14 & 90.73 ± 0.02   \\
% &EF21-SGD & 87.37 ± 0.49 & 65.82 ± 4.86 \\
% &NEOLITHIC & \textbf{94.63 ± 0.09} & \textbf{91.43 ± 0.10}   \\
% \midrule
% \multirow{4}{*}{1\%}&MEM-SGD & 93.99 ± 0.11 & 89.68 ± 0.17  \\
% &Double-Squeeze & 93.54 ± 0.17 & 89.35 ± 0.04   \\
% &EF21-SGD & 67.78 ± 2.14 & 56.0 ± 2.257 \\
% &NEOLITHIC & \textbf{94.155 ± 0.10} & \textbf{89.82 ± 0.37}   \\
% \bottomrule
% \end{tabular}
% \vskip -10mm
% % \end{minipage}
% \end{sc}
% \end{small}
% \end{center}
% \label{table-gc-method-topk5}
% \end{table}

\begin{table}
\tablefontsize
\caption{\small Accuracy comparison with different algorithms on CIFAR-10 (8 workers, 5\% compression ratio).}

\begin{center}
\begin{small}
\begin{sc}

\begin{tabular}{cccc}
\toprule

Methods   & ResNet18     & ResNet20   \\

\midrule

PSGD & 93.99 ± 0.52 & 91.62 ± 0.13  \\ 
MEM-SGD & 94.35 ± 0.01 & 91.27 ± 0.08 \\
Double-Squeeze & 94.11 ± 0.14 & 90.73 ± 0.02   \\
EF21-SGD & 87.37 ± 0.49 & 65.82 ± 4.86 \\
NEOLITHIC & \textbf{94.63 ± 0.09} & \textbf{91.43 ± 0.10}   \\

\bottomrule
\end{tabular}
\vskip -10mm
% \end{minipage}
\end{sc}
\end{small}
\end{center}
\label{table-gc-method-topk5}
\end{table}
% \end{wraptable}
All experiments were repeated three times with different seeds. For \ours, we set $R=2$. Following previous works \cite{Tang2019DoubleSqueezePS}, we use top-$k$ compressor with different compression ratio to evaluate the performance of the aforementioned methods. As shown in Figure~\ref{fig:cifar10},  Tables~\ref{table-gc-method-topk5} and Table~\ref{table-gc-method-topk1}, \ours consistently outperforms other compression methods and reaches the similar performance to \psgd. It is worth noting that EF21-SGD, while
guaranteed to converge with milder assumptions than ours, does not provide competitive performance in deep learning tasks listed in Tables  \ref{table-gc-method-topk5}, \ref{table-gc-method-topk1}, and \ref{tab:cifar10_heter_data}. We find our reported result for EF21-SGD is consistent with Figure~7 (the left plot) in \cite{Fatkhullin2021EF21WB}.

\begin{table}[h!]
\tablefontsize
\caption{\small Accuracy comparison with different algorithms on CIFAR-10 (8 workers, 1\% compression ratio).}

\begin{center}
\begin{small}
\begin{sc}
\begin{tabular}{ccc}
\toprule

Methods   & ResNet18     & ResNet20   \\

\midrule

MEM-SGD & 93.99 ± 0.11 & 89.68 ± 0.17  \\
Double-Squeeze & 93.54 ± 0.17 & 89.35 ± 0.04   \\
EF21-SGD & 67.78 ± 2.14 & 56.0 ± 2.257 \\
NEOLITHIC & \textbf{94.155 ± 0.10} & \textbf{89.82 ± 0.37}   \\

\bottomrule
\end{tabular}
\end{sc}
\end{small}
\end{center}
\label{table-gc-method-topk1}
\end{table}

\paragraph{The effect of compression ratio.} We also investigate the influence of different compression ratios. Table~\ref{table-gc-method-topk1} is with a compression ratio of 1\%, which indicates a harsher setting for compression methods. It is observed that \ours still outperforms other compression methods.

\paragraph{Performance with heterogeneous data. } We simulate data heterogeneity among workers via a Dirichlet distribution-based partition with parameter $\alpha$ controlling the data heterogeneity. The training data for a particular class tends to concentrate in a single node as $\alpha\rightarrow 0$, \emph{\ie} becoming more heterogeneous, while the homogeneous data distribution is achieved as $\alpha\rightarrow \infty$. We test $\alpha=1$ and $10$ in Table \ref{tab:cifar10_heter_data} for all compared methods as corresponding to a setting with high/low heterogeneity.

\begin{table}[t!]
    \centering
    \caption{\small Accuracy comparison on \cifar with heterogeneous data (ResNet-20).} 
    \vskip -1mm    
    \label{tab:cifar10_heter_data}
    \begin{small}
\begin{sc}
\setlength{\tabcolsep}{1.05mm}{
\begin{tabular}{cccccc}
\toprule
Methods  &   \memsgd & \ds & \ef & \ours  & C-Ratio \%       \\
\midrule
$\alpha=1$  & 85.42 ± 0.22 & 83.72 ± 0.17  & 38.5 ± 2.17 &  \textbf{85.47 ± 0.12} & 1 \\
$\alpha=10$  & 91.61 ± 0.19  & 91.25 ± 0.17 & 68.58 ± 0.13 &  \textbf{91.76 ± 0.13} &  1 \\

$\alpha=1$  & 86.43 ± 0.38  & 86.13 ± 0.21 & 72.65 ± 0.37 &  \textbf{87.17 ± 0.24} & 5 \\
$\alpha=10$  & 91.88 ± 0.23 & 91.66 ± 0.11 & 86.36 ± 0.15 &  \textbf{92.14 ± 0.22} &  5 \\
\bottomrule
\end{tabular}
}
\end{sc}
\end{small} 
\end{table}

\paragraph{Effects of accumulation rounds. } We also empirically evaluate the performance of NEOLITHIC with different choice of parameter $R$  in deep learning tasks. NEOLITHIC have slightly performance degradation in both compression scenarios as $R$ scales up. We conjecture that the gradient accumulation step, which amounts to using large-batch samples in gradient evaluation, can help in the optimization and training stage as proved in this paper, but it may hurt the generalization performance. We recommend using \ours in applications that are friendly to large-batch training.

\begin{table}[h!]
\centering

\caption{Effects of round numbers for \cifar dataset with ResNet-18}
\begin{small}
\begin{sc}
\setlength{\tabcolsep}{0.95mm}{
\begin{tabular}{ccccc}

\toprule
Rounds  & 2 & 3 & 4 & 5    \\
\midrule
% Ours & 93.75±.12  & 92.72±.20 & 92.43±.32 & 91.78±.28 \\ 

\ours(5\%) & 94.63 ± 0.09 & 93.32 ± 0.08  & 92.55 ± 0.12  & 91.48 ± 0.18   \\

\ours(1\%) & 94.16 ± 0.10 & 93.15 ± 0.11 & 92.27 ± 0.08  &  91.32 ± 0.12  \\
\bottomrule
\end{tabular}}
% \end{minipage}
\end{sc}
\end{small}
% \end{center}
\label{table-accu-rounds}
\end{table}

\section{Conclusion}
This paper provides lower bounds for distributed algorithms with communication compression, whether the compression is unidirectional or bidirectional and unbiased or contractive. An algorithm called NEOLITHIC is introduced to match the lower bounds under the assumption of bounded gradient dissimilarity. Future directions include developing optimal algorithms without the assumption, as well as discovering additional compression properties that might produce a better lower bound.

\section*{Acknowledgements}
The authors are grateful to Professor Peter Richtarik from KAUST for the helpful discussions in the relation between FCC and EF21 and various other useful suggestions. 

\bibliography{reference}
\bibliographystyle{abbrv}

\allowdisplaybreaks
\appendix
\section{Lower Bounds}\label{app:lower-bounds}
In this section, we provide the proofs for Theorem \ref{thm:lower1} and \ref{thm:lower3}. To proceed, we first introduce the zero-respecting property and zero-chain functions.

\subsection{Zero-Respecting Algorithms}\label{app:zero-property}
Let $[x]_j$ denote the $j$-th coordinate of a vector $x\in\RR^d$ for $j=1,\dots,d$, and define $\prog(x)$ as
\begin{equation}\label{eqn:prog-def}
    \prog(x):=\begin{cases}
    0 & \text{if $x=0$};\\
    \max_{1\leq j\leq d}\{j:[x]_j\neq 0\}& \text{otherwise}.
    \end{cases}
\end{equation}
Similarly, for a set of points $\cX=\{x_1,x_2,\dots\}$, we define $\prog(\cX):=\max_{x\in\cX}\prog(x)$. Apparently, it holds that $\prog(\cX\cup\cY)=\max\{\prog(\cX),\prog(\cY)\}$ for any $\cX,\,\cY\subseteq\RR^d$, and $\prog(\cX)\leq \prog(\widetilde{\cX})$ for any $\cX\subseteq\widetilde{\cX}\subseteq\RR^d$.

Now we consider the setup of distributed learning with communication compression. For each worker $i$ and $t\geq 1$, we let $y^{(t)}_i$ be the $t$-th variable at which worker $i$ queries its gradient oracle (with independent randomness $\zeta^{(t)}_i$) during the optimization procedure. 
We also let $x^{(t)}_i$ be the local model that worker $i$ produces after the $t$-th query. Note that $y^{(t)}_i$ is not necessarily the last-step local model $x^{(t-1)}_i$; it can be any auxiliary vector. 

Between the $(t-1)$-th and $t$-th gradient queries, each worker is allowed to communicate with the server by transmitting (compressed) vectors. For worker $i$, we let $\cV^{(t,\star)}_{\mathrm{w}_i\rightarrow \mathrm{s}}$ denote the set of vectors
that  worker $i$ aims to send to the server, \emph{\ie}, the vectors before compression. 
Due to  communication compression, the vectors received by the server from worker $i$, which we denote by $\cV^{(t)}_{\mathrm{w}_i\rightarrow \mathrm{s}}$, are the compressed version of $\cV^{(t,\star)}_{\mathrm{w}_i\rightarrow \mathrm{s}}$ with some underlying compressors $C_i$, \emph{\ie}, $\cV^{(t)}_{\mathrm{w}_i\rightarrow \mathrm{s}}\triangleq C_i(\cV^{(t,\star)}_{\mathrm{w}_i\rightarrow \mathrm{s}})$.
Note that $\cV^{(t)}_{\mathrm{w}_i\rightarrow \mathrm{s}}$ is a set that may include multiple vectors, and its cardinality equals the rounds of communication.
Here, for a compressor $C$ and a set of vectors $\cV^{(\star)}$, we use $C(\cV^{(\star)})$ to indicate the ``vector-wise'' compression over each vector in $\cV^{(\star)}$. After receiving the compressed vectors from all workers, the server will broadcast  some vectors back to all workers. We let   $\cV^{(t,\star)}_{\mathrm{s}\rightarrow \mathrm{w}}$  and  $\cV^{(t)}_{\mathrm{s}\rightarrow \mathrm{w}}$ denote the set of vectors that the server aims to send to workers and the ones received by workers, respectively. When algorithms conduct unidirectional compression, we always have $\cV^{(t)}_{\mathrm{s}\rightarrow \mathrm{w}}=\cV^{(t,\star)}_{\mathrm{s}\rightarrow \mathrm{w}}$. When bidirectional compression is conducted, we have $\cV^{(t)}_{\mathrm{s}\rightarrow \mathrm{w}}=C_0(\cV^{(t,\star)}_{\mathrm{s}\rightarrow \mathrm{w}})$ and $\cV^{(t)}_{\mathrm{s}\rightarrow \mathrm{w}}\neq \cV^{(t,\star)}_{\mathrm{s}\rightarrow \mathrm{w}}$ generally, depending on the behavior of the server-equipped compressor $C_0$.

Following the above description, we now extend the zero-respecting property \cite{Carmon2020LowerBF,Carmon2021LowerBF}, which firstly appears in single-node stochastic optimization, to distributed learning with communication compression:
\begin{definition}[Zero-respecting]\label{def:zero-repsect}
    We say a distributed algorithm $A$ is zero-respecting if for any $t\geq 1$ and $1\leq k\leq d$, the following requirements are satisfied:
\begin{enumerate}
    \item If worker $i$ queries at $y^{(t)}_i$ with $[y^{(t)}_i]_k\neq 0$, then one of the following must be true: 
    \begin{equation*}
        \begin{cases}
            \text{there exists some  $0\leq s<t$ such that $[x^{(s)}_i]_k\neq 0$};\\
            \text{there exists some $1\leq s< t$ such that $[O_i(y^{(s)}_i;\zeta^{(s)}_i)]_k\neq 0$};\\
            \text{there exists some $1\leq s< t$ such that worker $i$ has received some $v\in\cV_{\mathrm{s}\rightarrow\mathrm{w}}^{(s)}$ with $[v]_k\neq 0$};\\
            \text{there exists some $1\leq s<t$ such that worker $i$ has  compressed some  $v\in\cV_{\mathrm{w}_i\rightarrow\mathrm{s}}^{(s)}$ with $[v]_k\neq 0$}.
        \end{cases}
    \end{equation*}
    \item If the local model $x^{(t)}_i$ of worker $i$, after the $t$-th query, has $[x^{(t)}_i]_k\neq 0$, then one of the following must be true: 
    \begin{equation*}
        \begin{cases}
            \text{there exists some $0\leq s<t$ such that $[x^{(s)}_i]_k\neq 0$};\\
            \text{there exists some $1\leq s\leq t$ such that $[O_i(y^{(s)}_i;\zeta^{(s)}_i)]_k\neq 0$};\\
            \text{there exists some $1\leq s\leq t$ such that worker $i$ has received some $v\in\cV_{\mathrm{s}\rightarrow\mathrm{w}}^{(s)}$ with $[v]_k\neq 0$};\\
            \text{there exists some $1\leq s\leq t$ such that worker $i$ has  compressed some  $v\in\cV_{\mathrm{w}_i\rightarrow\mathrm{s}}^{(s)}$ with $[v]_k\neq 0$}.
        \end{cases}
    \end{equation*}
    \item If worker $i$ aims to send some $v\in\cV_{\mathrm{w}_i\rightarrow\mathrm{s}}^{(t,\star)}$ with $[v]_k\neq 0$ to the server,  then one of the following must be true: 
    \begin{equation*}
        \begin{cases}
            \text{there exits some $0\leq s<t$ such that $[x_i^{(s)}]_k\neq0$};\\
            \text{there exists some $1\leq s\leq  t$ such that $[O_i(y^{(s)}_i;\zeta^{(s)}_i)]_k\neq 0$};\\
            \text{there exists some $1\leq s< t$ such that worker $i$ has received some $v^\prime\in\cV_{\mathrm{s}\rightarrow\mathrm{w}}^{(s)}$ with $[v^\prime]_k\neq 0$};\\
            \text{there exists some $1\leq s<t$ such that worker $i$ has  compressed some  $v^\prime\in\cV_{\mathrm{w}_i\rightarrow\mathrm{s}}^{(s)}$ with $[v^\prime]_k\neq 0$}.
        \end{cases}
    \end{equation*}
    \item If the server aims to broadcast some $v\in\cV_{\mathrm{s}\rightarrow\mathrm{w}}^{(t,\star)}$ with $[v]_k\neq 0$ to workers, then one of the following must be true:
    \begin{equation*}
        \begin{cases}
            \text{there exists some $1\leq s\leq  t$ and $1\leq i\leq n$ such that }\\
            \text{\qquad \qquad the server has received some $v^\prime \in\cV_{\mathrm{w}_i\rightarrow\mathrm{s}}^{(s)}  $ with $[v^\prime]_k\neq 0$};\\
            \text{there exits some $1\leq s<t$ such that}\\
            \text{\qquad \qquad  the server has compressed and broadcast some  $v^\prime \in\cV_{\mathrm{s}\rightarrow\mathrm{w}}^{(s)} $  with $[v^\prime]_k\neq 0$}.
        \end{cases}
    \end{equation*}
\end{enumerate}
\end{definition}

In essence, the above zero-respecting property requires that any expansion of non-zero coordinates in $x^{(t)}_i$, $y^{(t)}_i$, and other related vectors in worker $i$ is  attributed to its historical local gradient updates, local compressions, or synchronization with the server.
Meanwhile, it also requires that any expansion of non-zero coordinate in vectors held in the server is due to its own compression operation, or receiving compressed messages from workers.

\begin{algorithm}[t]
	\caption{MEM-SGD}
	\label{alg:mem-sgd}
	\begin{algorithmic}
		\STATE \noindent {\bfseries Input:} Initialize $x^{(0)}$; compressors $\{C_i:1\leq i\leq n\}$; learning rate $\gamma$; $\er_i^{(0)}=0$, $\forall\, i\in[n]$
		\FOR{$t=0,1,\cdots, T$}
		\STATE  {\bfseries On all workers in parallel:}
		\STATE \quad Query stochastic gradient $\hat{g}_i^{(t)}=O_i(x^{(t)};\zeta_i^{(t)})$
		\STATE \quad Compress and send error compensated message $\tilde{g}_i^{(t)}=C_i(\gamma \hat{g}_i^{(t)}+\er_i^{(t)})$ to the server
		\STATE  {\bfseries On server:}
		\STATE \quad Aggregate messages and broadcast $\tilde{g}^{(t)}=\frac{1}{n}\sum_{i=1}^n\tilde{g}_i^{(t)}$ to all workers
		\STATE  {\bfseries On all workers in parallel:}
		\STATE \quad Update model parameter $x^{(t+1)}=x^{(t)}-\tilde{g}^{(t)}$ 
		\STATE \quad Update accumulated error $\er_i^{(t+1)}=\er_i^{(t)}+\gamma \hat{g}_i^{(t)}- \tilde{g}_i^{(t)}$
		\ENDFOR
	\end{algorithmic}
\end{algorithm}

To better illustrate the zero-respecting property, we take MEM-SGD \cite{Stich2018SparsifiedSW} as an example and show how it satisfies the zero-respecting property. MEM-SGD is listed in Algorithm \ref{alg:mem-sgd}, which conducts unidirectional compression. Following the notations in Definition \ref{def:zero-repsect}, for any $t\geq 1$ and $1\leq i\leq n$, we have for MEM-SGD
\begin{equation}
    y_i^{(t)}=x_i^{(t-1)}=x^{(t-1)};\quad \cV_{\mathrm{w}_i\rightarrow\mathrm{s}}^{(t,\star)}=\{\gamma \hat{g}_i^{(t-1)}+\er_i^{(t-1)}\}\quad \text{and}\quad \cV_{\mathrm{w}_i\rightarrow\mathrm{s}}^{(t)}=\{\tilde{g}_i^{(t-1)}\};\quad \cV_{\mathrm{s}\rightarrow\mathrm{w}}^{(t,\star)}= \cV_{\mathrm{s}\rightarrow\mathrm{w}}^{(t)}=\{\tilde{g}^{(t-1)}\}.
\end{equation}
Now we verify  Definition \ref{def:zero-repsect} one by one. 
\begin{enumerate}
    \item For point 1, for any $1\leq k\leq d$, $[y^{(t)}_i]_k\neq 0$ is equivalent to $[x^{(t-1)}_i]_k\neq 0$.
    \item For point 2, since $x^{(t)}_i=x^{(t)}=x^{(t-1)}-\gamma \tilde{g}^{(t-1)}=x^{(t-1)}_i-\gamma \tilde{g}^{(t-1)}$, where $\tilde{g}^{(t-1)}\in \cV_{\mathrm{s}\rightarrow\mathrm{w}}^{(t)}$ is received from the server, we have $[x^{(t)}_i]_k=[x^{(t-1)}_i]_k-\gamma [\tilde{g}^{(t-1)}]_k$. Therefore, $[x^{(t)}_i]_k\neq 0$ implies that one of $[x^{(t-1)}_i]_k$ and $[\tilde{g}^{(t-1)}]_k$ is non-zero.
    \item For point 3, it is easy to see that $\er_i^{(t-1)}=\sum_{s=0}^{t-2}(\gamma \hat{g}_i^{(s)}-\tilde{g}_i^{(s)})$ and we thus have $\gamma \hat{g}_i^{(t-1)}+\er_i^{(t-1)}=\gamma\sum_{s=0}^{t-1} \hat{g}_i^{(s)}-\sum_{s=0}^{t-2}\tilde{g}_i^{(s)}$, where $\hat{g}_i^{(s)}=O_i(x^{(s)};\zeta_i^{(s)})$ is a local gradient query and  $\tilde{g}_i^{(s)}\in \cV_{\mathrm{w}_i\rightarrow\mathrm{s}}^{(s)}$ is due to the internal compression within worker $i$.
     Therefore, $[\gamma \hat{g}_i^{(t-1)}+\er_i^{(t-1)}]_k\neq 0$ implies that
     one element in $\{[O_i(x^{(s)};\zeta_i^{(s)})]_k:0\leq s\leq t-1\}$ or $\{[\tilde{g}_i^{(s)}]_k:0\leq s\leq t-2\}$
     is non-zero.
    \item For point 4, since $\tilde{g}^{(t-1)}=\frac{1}{n}\sum_{i=1}^n\tilde{g}_i^{(t-1)}$ where  $\tilde{g}_i^{(t-1)}$ is received from worker $i$, we have $[\tilde{g}^{(t-1)}]_k\neq 0$, implying that one element in  $\{[\tilde{g}_i^{(t-1)}]_k:1\leq i\leq n\}$ is non-zero.
\end{enumerate}

While we only validate the zero-respecting property in MEM-SGD, one can easily follow the above procedure to verify  that other algorithms, including those listed in Table \ref{tab:introduction-comparison}, are zero-respecting.
In fact, to our knowledge, the zero-respecting property is satisfied by all existing distributed algorithms with communication compression.

\subsection{Zero-Chain Functions}
We next introduce zero-chain functions, which can be combined with the zero-respecting property discussed in Appendix \ref{app:zero-property} to prove the lower bounds for distributed learning under communication compression.
As described in \cite{Carmon2020LowerBF,Carmon2021LowerBF}, a zero chain function $f$ satisfies
\begin{equation*}
    \prog(\nabla f(x))\leq \prog(x)+1,\quad\forall\,x\in\RR^d,
\end{equation*}
which implies that, starting from $x=0$, a single gradient evaluation can only make at most one more coordinate of the model parameter $x$ be non-zero.

We state some key zero-chain functions that will be used to facilitate the analysis.
\begin{lemma}[Lemma 2 of \cite{Arjevani2019LowerBF}]\label{lem:basic-fun}
Let $[x]_j$  denote the $j$-th coordinate of a vector $x\in\RR^d$, and define function 
\begin{equation*}
    h(x):=-\psi(1) \phi([x]_{1})+\sum_{j=1}^{d-1}\Big(\psi(-[x]_j) \phi(-[x]_{j+1})-\psi([x]_j) \phi([x]_{j+1})\Big)
\end{equation*}
where for $\forall\, z \in \mathbb{R},$
$$
\psi(z)=\begin{cases}
0 & z \leq 1 / 2; \\
\exp \left(1-\frac{1}{(2 z-1)^{2}}\right) & z>1 / 2, 
\end{cases} \quad \quad \mbox{and} \quad \quad  \phi(z)=\sqrt{e} \int_{-\infty}^{z} e^{-\frac{1}{2} t^{2}} \mathrm{d}t.
$$
The function $h(x)$ satisfies the following properties:
\begin{enumerate}
    \item $h$ is zero-chain, \emph{\ie}, $\prog(\nabla h(x))\leq \prog(x)+1$ for all $x\in\RR^d$.
    \item $h(x)-\inf_{x} h(x)\leq \Delta_0 d$, $\forall\,x\in\RR^d$ with $\Delta_0=12$.
    \item $h$ is $L_0$-smooth with $L_0=152$.
    \item $\|\nabla h(x)\|_\infty\leq G_\infty $, $\forall\,x\in\RR^d$ with $G_\infty = 23$.
    \item $\|\nabla h(x)\|_\infty\ge 1 $ for any $x\in\RR^d$ with $[x]_d=0$. 
\end{enumerate}
\end{lemma}

\begin{lemma}\label{lem:basic-fun2}
If we let functions 
\begin{equation*}
    h_1(x):=-2\psi(1) \phi([x]_{1})+2\sum_{j \text{ even, } 0< j<d}\Big(\psi(-[x]_j) \phi(-[x]_{j+1})-\psi([x]_j) \phi([x]_{j+1})\Big)
\end{equation*}
and 
\begin{equation*}
    h_2(x):=2\sum_{j \text{ odd, } 0<j<d}\Big(\psi(-[x]_j) \phi(-[x]_{j+1})-\psi([x]_j) \phi([x]_{j+1})\Big),
\end{equation*}
then $h_1$ and $h_2$ satisfy  the following properties:
\begin{enumerate}
    % \item Both $h_1$ and $h_2$ are zero-chain functions.
    \item $\frac{1}{2}(h_1+h_2)=h$, where $h$ is defined in Lemma \ref{lem:basic-fun}.
    \item $h_1$ and $h_2$ are zero-chain, \emph{\ie}, $\prog(\nabla h_i(x))\leq \prog(x)+1$ for all $x\in\RR^d$ and $i=1,2$. Furthermore, if $\prog(x)$ is odd, then $\prog(\nabla h_1(x))\leq \prog(x)$; if $\prog(x)$ is even, then $\prog(\nabla h_2(x))\leq \prog(x)$.
    \item $h_1$ and $h_2$ are also $L_0$-smooth with ${L_0}=152$. 
\end{enumerate}
\end{lemma}
\begin{proof}
The first property follows the definition of $h_1$, $h_2$, and $h$. The second property follows Lemma 1 of \cite{Carmon2020LowerBF} that $\psi^{(m)}(0)=0$ for any $m\in\mathbb{N}$. Now we prove the third property. Noting that the Hessian of $h_k$ for $k=1,2$ is tridiagonal and symmetric, we have, by the Schur test, for any $x\in\RR^d$ and $k=1,2$ that 
\begin{align}
    \|\nabla^2 h_k(x)\|_2\leq&\sqrt{\|\nabla^2 h_k(x)\|_1\|\nabla^2 h_k(x)\|_\infty}=\|\nabla^2 h_k(x)\|_1.\label{eqn:jvoqfqfvcdv}
\end{align}
Furthermore, it is easy to verify that 
\begin{align}
    &\|\nabla^2 h_k(x)\|_1\nonumber\\
    \leq& 2\max\left\{\sup _{z \in \mathbb{R}}\left|\psi^{\prime \prime}(z)\right| \sup _{z \in \mathbb{R}}|\phi(z)|,\sup _{z \in \mathbb{R}}\left|\psi(z)\right| \sup _{z \in \mathbb{R}}|\phi^{\prime \prime}(z)|\right\}+2 \sup _{z \in \mathbb{R}}\left|\psi^{\prime}(z)\right| \sup _{z \in \mathbb{R}}\left|\phi^{\prime}(z)\right|\leq 152,\label{eqn:vjowemfqvgdwv}
\end{align}
where $\psi^{\prime}(z)$ and $\psi^{\prime \prime}(z)$ are the first- and second-order derivative of $\psi(z)$, respectively, $\phi^{\prime}(z)$ and $\phi^{\prime \prime}(z)$ are the first- and second-order derivative of $\phi(z)$, respectively, and the last inequality follows Observation 2 of \cite{Arjevani2019LowerBF} that 
\begin{equation*}
    0 \leq \psi \leq e, \;\; 0 \leq \psi^{\prime} \leq \sqrt{54 / e}, \;\;\left|\psi^{\prime \prime}\right| \leq 32.5,\;\; 0 \leq \phi \leq \sqrt{2 \pi e}, \;\; 0 \leq \phi^{\prime} \leq \sqrt{e}\text{ \;and }\left|\phi^{\prime \prime}\right| \leq 1.
\end{equation*}
Combining \eqref{eqn:vjowemfqvgdwv} with \eqref{eqn:jvoqfqfvcdv}, we know $h_k$ is $L_0$-smooth for $k=1,2$.
\end{proof}

\subsection{Proof of Theorem \ref{thm:lower1}}\label{app:proof-lb-1}
Without loss of generality, we assume algorithms to start from $x^{(0)}=0$.
We prove the two terms in the lower bound established in Theorem \ref{thm:lower1} separately by constructing two hard-to-optimize examples.
The construction of the example, for each term in \eqref{eqn:lower-bound1},  can be conducted in three steps: 1) constructing local functions $\{f_i\}_{i=1}^n$ by following Lemma \ref{lem:basic-fun} and \ref{lem:basic-fun2}; 2) constructing compressors $\{C_i\}_{i=1}^n\subseteq \cU_\omega$ and independent oracles $\{O_i\}_{i=1}^n\subseteq \cO_{\sigma^2}$ that hamper algorithms to expand the non-zero coordinates of model parameters; 3) establishing a limitation, in terms of the non-zero coordinates of model parameters, 
for zero-respecting algorithms utilizing the predefined compression protocol 
% obeying the desired compression protocol 
with $T$ gradient queries and compressed communications on each worker, and finally we translate this limitation into the lower bound of convergence rate.

\paragraph{Example 1. } 
The proof of the first term $\Omega((\frac{\Delta L\sigma^2}{nT})^\frac{1}{2})$ is adapted from the first example in proving Theorem 1 of \cite{Lu2021OptimalCI}. 

(Step 1.) Let $f_i=L\lambda^2h(x/\lambda)/L_0$, $\forall\,i=1,\dots,n$ be homogeneous and hence $f=L\lambda^2h(x/\lambda)/L_0$ where $h$ is defined in Lemma \ref{lem:basic-fun}  and $\lambda>0$ is to be specified. Since $\nabla^2 f_i=L\nabla^2 h/L_0$ and $h$ is $L_0$-smooth (Lemma \ref{lem:basic-fun}), we know $f_i$ is $L$-smooth for any $\lambda>0$. By Lemma \ref{lem:basic-fun}, we have
\begin{equation*}
    f(0)-\inf_x f(x)=\frac{L\lambda^2}{L_0}(h(0)-\inf_x h(x)) {\leq}\frac{L\lambda^2\Delta_0d}{L_0}. 
\end{equation*}
Therefore, to ensure $f_i\in\cF_{\Delta, L}$, it suffices to let 
\begin{equation}\label{eqn:jgowemw}
    \frac{L\lambda^2\Delta_0d}{L_0}\leq \Delta, \quad \text{\emph{\ie},}\quad d\lambda^2\leq \frac{L_0 \Delta}{L\Delta_0}.
\end{equation}
% where the right-hand side of \eqref{eqn:jgowemw} is a absolute constant.

(Step 2.) We assume all compressors $\{C_i\}_{i=1}^n$ to be identities, meaning that there is no compression error in the entire optimization procedure. It naturally follows that  $\{C_i\}_{i=1}^n\subseteq \cU^\omega$ for any $\omega\geq 0$. We construct the stochastic gradient oracle $O_i$ on worker $i$, $\forall\,i=1,\dots,n$ as follows:
\begin{equation*}
    [O_i(x;\zeta)]_j=[\nabla f_i(x)]_j\left(1+\mathds{1}{\{j>\prog(x)\}}\left(\frac{\zeta}{p}-1\right)\right), \forall\,x\in\RR^d,\,j=1,\dots,d
\end{equation*}
with random variable $\zeta\sim \text{Bernoulii}(p)$ independent of $x$ and $i$, and $p\in(0,1)$ is to be specified.
The oracle $O_i$ has a chance 
% probability 
to make the entry $[\nabla f_i(x)]_{\prog(x)+1}$ zero. Therefore, algorithms are hampered by oracle $O_i$ to reach more non-zero coordinates. Formally we have 
\begin{equation}\label{eqn:gjoemoqwfq}
    \prog(O_i(x;\zeta))\leq \prog(x)+\mathds{1}(\{\zeta=1\})
\end{equation}
where $\mathds{1}(E)$ is the indicator of event $E$.

It is easy to see that $O_i$ is an unbiased stochastic gradient oracle. Moreover, since $f_i$ is zero-chain, by \eqref{eqn:gjoemoqwfq} and the definition of $O_i$,  we have
\begin{align*}
\EE[\|[O_i(x;\zeta)]-\nabla f_i(x)\|^2]&=|[\nabla f_i(x)]_{\prog(x)+1}|^2\EE\left[\left(\frac{\zeta}{p}-1\right)^2\right]=|[\nabla f_i(x)]_{\prog(x)+1}|^2\frac{1-p}{p}\\
&\leq \|\nabla f_i(x)\|_\infty^2\frac{1-p}{p}
\leq \frac{L^2\lambda^2(1-p)}{L_0^2p}\|\nabla h(x)\|_\infty^2\\
&\overset{\text{Lemma \ref{lem:basic-fun}}}{\leq } \frac{L^2\lambda^2(1-p)G_\infty^2}{L_0^2p}.
\end{align*}
Therefore, to ensure $O_i\in\cO_{\sigma^2}$, it suffices to let 
\begin{equation}\label{eqn:gjvownefoq}
     p=\min\{\frac{L^2\lambda^2G_\infty^2}{L_0^2\sigma^2},1\}.
\end{equation}

(Step 3.) 
Let $y^{(t)}_i$, $\forall\,t\geq 1$ and $1\leq i\leq n$, be the $t$-th variable at which worker $i$ queries gradient from oracle $O_i$, and let $x^{(t)}_i$ be the local model it
produces after the $t$-th query. 
Let 
\begin{equation*}
    \hat{x}\in\mathrm{span}\left(\left\{x^{(t)}_i:0\leq t\leq T,1\le i\leq n\right\}\right),
\end{equation*}
be the final algorithm  output after $T$ gradient queries on each worker.
Since algorithms satisfy the zero-respecting property, as discussed in Appendix \ref{app:zero-property}, it holds that  $[\hat{x}]_k\neq 0$ implies at least one $[x^{(t)}_i]_k$, where $0\le t \le T$, is non-zero, which further implies at least one $[O_i(y^{(t)}_i;\zeta_i^{(t)})]_k$ is non-zero. Therefore, by considering the index of the last non-zero coordinate of $\hat{x}$ and using \eqref{eqn:gjoemoqwfq}, we have
\begin{equation}\label{eqn:ghiowfnowngowfgq}
    \prog(\hat{x})\leq \max_{1\leq t\leq T}\max_{1\leq i\leq n}\prog(O_i(y^{(t)}_i;\zeta_i^{(t)}))\leq  \max_{1\leq t\leq T}\max_{1\leq i\leq n}\{\prog(y^{(t)}_i)+\mathds{1}(\{\zeta_i^{(t)}=1\})\}.
\end{equation}
Similarly, considering that the non-zero coordinates of $y^{(t)}_i$ can be tracked back to the non-zero coordinates of past gradient queries and using \eqref{eqn:ghiowfnowngowfgq}, we have
\begin{equation}\label{eqn:gjownfrownfrqfvd}
    \prog(y^{(t)}_i)\leq \max_{1\leq s< t}\max_{1\leq i\leq n}\prog(O_i(y^{(s)}_i;\zeta_i^{(s)}))\leq \max_{1\leq s< t}\max_{1\leq i\leq n}\{\prog(y^{(s)}_i)+\mathds{1}(\{\zeta_i^{(s)}=1\})\}.
\end{equation}
By induction for \eqref{eqn:gjownfrownfrqfvd} with respect to $t=1,\dots,T$, we reach 
\begin{equation}
    \max_{1\leq t\leq T}\max_{1\leq i\leq n}\prog(y^{(t)}_i)\leq\sum_{t=1 }^{T-1}\max_{1\leq i\leq n} \mathds{1}(\{\zeta_i^{(t)}=1\}),
\end{equation}
which, combined with \eqref{eqn:ghiowfnowngowfgq}, leads to 
\begin{equation}
     \prog(\hat{x})\leq \sum_{t=1 }^{T}\max_{1\leq i\leq n} \mathds{1}(\{\zeta_i^{(t)}=1\}).
\end{equation}
By Lemma 2 of \cite{Lu2021OptimalCI}, which is derived from the Chernoff bound, we have 
\begin{equation}\label{eqn:vjowemnfq}
    \PP( \prog(\hat{x})\ge d)\leq \PP\left(\sum_{t=1 }^{T}\max_{1\leq i\leq n} \mathds{1}(\{\zeta_i^{(t)}=1\})\geq d\right)\leq e^{(e-1)npT-d}.
\end{equation}
On the other hand, when $\prog(\hat{x})< d$, it holds by the fourth point in Lemma \ref{lem:basic-fun} that 
\begin{align}\label{eqn:vjowemnfq2}
    \min_{\hat{x}\in\mathrm{span}(\{x_i^{(t)}:1\leq i\leq n,\,0\leq t\leq  T\})}\|\nabla f(\hat{x})\|\geq \min_{[\hat{x}]_{d}=0}\|\nabla f(\hat{x})\|
    =\frac{L\lambda}{L_0}\min_{[\hat{x}]_{d}=0}\|\nabla h(\hat{x})\|\geq \frac{L\lambda}{L_0}.
\end{align}                                                       
Therefore, by combining \eqref{eqn:vjowemnfq} and \eqref{eqn:vjowemnfq2}, we have
\begin{align}
    \EE[\|\nabla f(\hat{x})\|^2]\geq &\PP( \prog(\hat{x})< d)\EE[\|\nabla f(\hat{x})\|^2\mid  \prog(\hat{x})< d]\label{eqn:Lvjowenfq}\\
    \geq& (1-e^{(e-1)npT-d})\frac{L^2\lambda^2}{L_0^2}.\nonumber
\end{align}

If we let 
\begin{equation}
    \lambda=\frac{L_0}{L}\left(\frac{\Delta L\sigma^2}{3nT  L_0\Delta_0 G_\infty^2}\right)^\frac{1}{4}\quad \text{and}\quad d=\left\lfloor\left(\frac{3L\Delta nT G_\infty^2}{\sigma^2L_0\Delta_0}\right)^\frac{1}{2}\right\rfloor,
\end{equation}
then \eqref{eqn:jgowemw} naturally holds and $p=\min\{\frac{G_\infty^2}{\sigma^2}\left(\frac{\Delta L\sigma^2}{3nT L_0\Delta_0 G_\infty^2}\right)^\frac{1}{2},1\}$ by \eqref{eqn:gjvownefoq}. Without loss of generality, we assume   $d\geq 2$, which is guaranteed when $T=\Omega(\frac{\sigma^2}{nL\Delta})$. Thus, using the definition of $p$, we have that
\begin{align}
    (e-1)npT-d\leq& (e-1)nT\; \frac{G_\infty^2}{\sigma^2}\left(\frac{\Delta L\sigma^2}{3nTL_0\Delta_0 G_\infty^2}\right)^\frac{1}{2}-d\nonumber\\
    =&\frac{e-1}{3}\left(\frac{3 L\Delta nT G_\infty^2}{\sigma^2 L_0\Delta_0 }\right)^\frac{1}{2}-d< \frac{e-1}{3}(d+1)-d\leq  1-e<0\label{eqn:Lvjobqwm}\nonumber
\end{align}
which, combined with \eqref{eqn:Lvjowenfq} and that $\Delta_0,\,L_0,\,G_\infty$ are universal constants, leads to
\begin{equation*}
    \EE[\|\nabla f(\hat{x})\|^2]=\Omega\left(\frac{L^2\lambda^2}{L_0^2}\right)=\Omega\left(\left(\frac{\Delta L\sigma^2}{3nT L_0\Delta_0 G_\infty^2}\right)^\frac{1}{2}\right)=\Omega\left(\left(\frac{\Delta L\sigma^2}{nT}\right)^\frac{1}{2}\right).
\end{equation*}

\paragraph{Example 2.} Without loss of generality, we assume $n$ is even; otherwise we can consider the lower bound for the case of $n-1$.

(Step 1.) Similar to but different from the construction of Example 1, we let $f_i=L\lambda^2 h_1(x/\lambda)/L_0$, $\forall\,1\leq i\leq n/2$ and $f_i=L\lambda^2 h_2(x/\lambda)/L_0$, $\forall\,n/2<i\leq n$, where $h_1$ and $h_2$ are defined in Lemma \ref{lem:basic-fun2}, and $\lambda>0$ will be specified later. By the definitions of $h_1$ and $h_2$, we have that $f_i$, $\forall\,1\leq i\leq n$, is  zero-chain  and $f(x)=\frac{1}{n}\sum_{i=1}^n f_i(x)=L\lambda^2 h(x/\lambda)/L_0$. Since $h_1$ and $h_2$ are also $L_0$-smooth, to let $f\in\cF_{\Delta,L}$, it suffices to make \eqref{eqn:jgowemw} hold.

In the above construction, we essentially split a zero-chain
function, \emph{\ie}, $h$, into two different components: the even component of the chain, \emph{\ie}, $h_1$, and the odd component of the chain, \emph{\ie}, $h_2$. Recall the second point in Lemma \ref{lem:basic-fun2} that for any $x\in\RR^d$, if $\prog(x)$ is odd, then $\prog(\nabla h_1(x))\leq \prog(x)$; if $\prog(x)$ is even, then $\prog(\nabla h_2(x))\leq \prog(x)$. Therefore, the workers, starting from any point with any algorithm $A\in\cA_{\{C_i\}_{i=1}^n}^U$, can only earn one more non-zero coordinate if they do not synchronize information (non-zero coordinates) with workers endowed with local functions of the other component; after that, the number of non-zero coordinates of local models will not increase any more. That is to say, in order to proceed (\emph{\ie}, achieve more non-zero coordinates), the worker must exchange the gradient information, via the server, between the odd and even components.

(Step 2.) 
We consider a gradient oracle that can return the full-batch gradient, \emph{\ie}, $O_i(x)=\nabla f_i(x)$, $\forall\,x\in\RR, \,1\leq i\leq n$, meaning that there is no gradient stochasticity  during the entire optimization procedure. For the construction of $\omega$-unbiased compressors, we consider $\{C_i\}_{i=1}^n$ to be the $\frac{d}{s}\times$rand-$s$ operators with shared randomness (\emph{\ie}, all workers share the same random seed) and $s=\lceil d/(1+\omega)\rceil$, where the $\frac{d}{s}$-scaling procedure is to ensure unbiasedness.
Specifically, during a round of communication compression, all workers randomly choose  $s$ coordinates from the full vector to be communicated, and then transmit the  $\frac{d}{s}$-scaled values at chosen coordinates. The chosen coordinate indexes are identical across all workers due to the shared randomness and are sampled uniformly randomly per communication. 
Since each coordinate index has probability $s/d$ to be chosen, we have  for any $x\in\RR^d$
\begin{equation*}
   \EE[C_i(x)]=\EE\left[\left(\frac{d}{s}x_k\mathds{1}\{k \text{ is chosen}\}\right)_{1\leq k\leq d}\right]=\left(\frac{d}{s}x_i\PP(k \text{ is chosen})\right)_{1\leq k\leq d}=x,
\end{equation*}
and 
\begin{align*}
   \EE[\|C_i(x)-x\|^2]=&\sum_{k=1}^d\EE\left[\left(\frac{d}{s}x_k\mathds{1}\{k \text{ is chosen}\}-x_k\right)^2\right]\\
   =&\sum_{k=1}^dx_k^2\left(\left(\frac{d}{s}-1\right)^2\PP(k \text{ is chosen})+\PP(k \text{ is not chosen})\right)=\sum_{k=1}^dx_k^2\left(\frac{d}{s}-1\right)\leq \omega \|x\|^2
\end{align*}
where the last inequality follows the definition of $s$. Therefore, the above construction gives $\{C_i\}_{i=1}^n\subseteq\cU_\omega$. 

(Step 3.) 
For any $t=1,\dots,T$ and $1\leq i\leq n$, let $v^{(t,\star)}_{\mathrm{w}_i\rightarrow \mathrm{s}}$ be the vector that worker $i$ aims to send to the server at the $t$-th communication. Due to communication compression, the server can only receive a compressed version $C_i(v^{(t,\star)}_{\mathrm{w}_i\rightarrow \mathrm{s}})$ instead, which we denote by $v^{(t)}_{\mathrm{w}_i\rightarrow \mathrm{s}}\triangleq C_i(v^{(t,\star)}_{\mathrm{w}_i\rightarrow \mathrm{s}})$. It is easy to see that 
\begin{equation}\label{eqn:fiqpjfqcv}
    \prog(v^{(t)}_{\mathrm{w}_i\rightarrow \mathrm{s}})\leq \prog(v^{(t,\star)}_{\mathrm{w}_i\rightarrow \mathrm{s}}),\quad \forall\, v^{(t,\star)}_i\in\RR^d.
\end{equation}
Furthermore, since the compressor $C_i$ only takes $s$ coordinates, $v^{(t)}_{\mathrm{w}_i\rightarrow \mathrm{s}}$ has probability $1-s/d\approx \frac{\omega}{1+\omega}$ in making the coordinate with index $\prog(v^{(t,\star)}_{\mathrm{w}_i\rightarrow \mathrm{s}})$ zero, which implies  $\PP(\prog(v^{(t)}_{\mathrm{w}_i\rightarrow \mathrm{s}})< \prog(v^{(t,\star)}_{\mathrm{w}_i\rightarrow \mathrm{s}}))\approx \frac{\omega}{1+\omega}$. This is to say, each worker $i$ has only probability $\approx (1+\omega)^{-1}$ to transmit its last non-zero entry. Therefore, algorithms are hampered by the compressors to communicate non-zero coordinates across all workers.

We let $v^{(t)}_{\mathrm{s}\rightarrow\mathrm{w}}$ be the vector that workers receive from the server in the $t$-th communication and let $x_i^{(t)}$ be the local model that worker $i$ produces after the $t$-th communication. Recall that algorithms satisfy the zero-respecting property discussed in Appendix \ref{app:zero-property}. Following the argument in Step 1 of Example 2, we find that each worker can only achieve one more non-zero coordinate in local models by
 local gradient updates based on the received messages from the server. Therefore, we have that
\begin{equation}\label{eqn:gjvowfnoqwmfgqw=1}
    \prog(x_i^{(t)})\leq \max_{1\leq s\leq t}\prog(v^{(s)}_{\mathrm{s}\rightarrow\mathrm{w}})+1.
\end{equation}
By further noting vector $v^{(t)}_{\mathrm{s}\rightarrow\mathrm{w}}$ sent by the server can be tracked back to past vectors received from all workers, we have
\begin{equation}\label{eqn:gjvowfnoqwmfgqw=2}
    \prog(v^{(t)}_{\mathrm{s}\rightarrow\mathrm{w}})\leq\max_{1\leq s\leq t} \max_{1\leq i\leq n}\prog(v^{(s)}_{\mathrm{w}_i\rightarrow\mathrm{s}}).
\end{equation}
Combining \eqref{eqn:gjvowfnoqwmfgqw=1} and \eqref{eqn:gjvowfnoqwmfgqw=2}, we reach 
\begin{equation}\label{eqn:ghoewjgoegqwegftqw}
    \prog(x_i^{(t)})\leq \max_{1\leq s\leq t}\max_{1\leq i\leq n}\prog(v^{(s)}_{\mathrm{w}_i\rightarrow\mathrm{s}})+1.
\end{equation}
Let 
\begin{equation*}
    \hat{x}\in\mathrm{span}\left(\left\{x^{(t)}_i:0\leq t\leq T,1\le i\leq n\right\}\right),
\end{equation*}
be the final algorithm  output after $T$ rounds of communication on each worker.
By \eqref{eqn:ghoewjgoegqwegftqw}, we have 
\begin{equation}\label{eqn:ghowgnowtqwrqwcx}
    \prog(\hat{x})\leq  \max_{1\leq t\leq T}\max_{1\leq i\leq n}\prog(v^{(t)}_{\mathrm{w}_i\rightarrow\mathrm{s}})+1.
\end{equation}

By Lemma \ref{lem:small-prob}, we have 
\begin{equation}\label{eqn:voqwefqb}
    \PP(\max_{1\leq t\leq T}\max_{1\leq i\leq n}\prog(v^{(t)}_{\mathrm{w}_i\rightarrow\mathrm{s}})\geq d-1)\leq e^{(e-1)T \lceil d/(1+\omega)\rceil/d+1-d}.
\end{equation}
Combining \eqref{eqn:voqwefqb}, \eqref{eqn:ghowgnowtqwrqwcx}, \eqref{eqn:vjowemnfq2}, and\eqref{eqn:Lvjowenfq}, we have that 
\begin{equation}\label{eqn:Lvjowenfq22}
    \EE[\|\nabla f(\hat{x})\|^2]\geq (1-e^{(e-1)T\lceil d/(1+\omega)\rceil/d+1-d})\frac{L^2\lambda^2}{L_0^2}.
\end{equation}
% where we denote $\hat{x}_i$ as the point with most non-zero coordinates that all workers and the server can attain as of the $t$-th communication.

If we let
\begin{equation*}
    \lambda =\frac{L_0}{L}\sqrt{\frac{(1+\omega)\Delta L}{5 TL_0\Delta_0}} \quad \text{and}\quad d= \lfloor 5 T/(1+\omega)\rfloor .
\end{equation*}
then \eqref{eqn:jgowemw} naturally holds.
Since $T$ is assumed to be no less than $(1+\omega)^2$, we have
$d=\lfloor 5T/(1+\omega)\rfloor \geq  5 T/(1+\omega)-1\geq 4 T/(1+\omega)\geq 4(1+\omega)\geq 4$.
Then it is easy to verify 
\begin{align}
    &(e-1)T \left\lceil  \frac{d}{1+\omega}\right\rceil /d+1-d \leq (e-1)T \left( \frac{d}{1+\omega}+1\right) /d+1-d\nonumber\\
    =&(e-1)\frac{T}{1+\omega}+(e-1)T/d+1-d\leq (e-1)\frac{T}{1+\omega}+(e-1)\frac{T}{4(1+\omega)}+1-\frac{4T}{1+\omega}\nonumber\\
    =&((e-1)(1+\frac{1}{4})-4)\frac{T}{1+\omega}+1 \leq \frac{5e-17}{4}<0,\nonumber
\end{align}
which, combined with \eqref{eqn:Lvjowenfq22} and that $\Delta_0,\,L_0$ are universal constants, leads to
\begin{equation*}
    \EE[\|\nabla f(\hat{x})\|^2]=\Omega\left(\frac{L^2\lambda^2}{L_0^2}\right)=\Omega\left(\frac{(1+\omega)\Delta L}{5T  L_0\Delta_0}\right)=\Omega\left(\frac{(1+\omega)\Delta L}{T}\right).
\end{equation*}

\begin{lemma}\label{lem:small-prob}
In Example 2 in the proof of Theorem \ref{thm:lower1}, it holds that
\begin{equation*}
    \PP(\max_{1\leq t\leq T}\max_{1\leq i\leq n}\prog(v^{(t)}_{\mathrm{w}_i\rightarrow\mathrm{s}})\geq d-1)\leq e^{(e-1)T \lceil d/(1+\omega)\rceil/d-d+1}
\end{equation*}
for any $T\geq 1$.
\end{lemma}
\begin{proof}
Note that at the $t$-th round of communication where $1\leq t\leq T$, the non-zero coordinates of $v^{(t,\star)}_{\mathrm{w}_i\rightarrow\mathrm{s}}$, the vector that is to be transmitted by worker $i$ to the server before compression, are achieved by utilizing previously received vectors $\{v^{(s)}_{\mathrm{s}\rightarrow\mathrm{w}}:1\leq s \leq t-1\}$ and local gradient queries. Following the argument in Step 1 and Step 3 of Example 2, we find that worker $i$ can only achieve one more non-zero coordinate in $v^{(t,\star)}_{\mathrm{w}_i\rightarrow\mathrm{s}}$ by
 local gradients updates based on received vectors $\{v^{(s)}_{\rm s\to w}:1\leq s \leq t-1\}$. 
Therefore, it holds that
\begin{align}\label{eqn:bvjowenf2q-dfhiafsa}
    \prog(v^{(t,\star)}_{\mathrm{w}_i\rightarrow\mathrm{s}})\leq \max_{1\leq s\leq  t-1}\prog(v^{(s)}_{\mathrm{s}\rightarrow\mathrm{w}})+1 \overset{\eqref{eqn:gjvowfnoqwmfgqw=2}}{\leq }\max_{1\leq s\leq t-1}\max_{1\leq i\leq n}\prog(v^{(s)}_{\mathrm{w}_i\rightarrow\mathrm{s}})+1\triangleq B^{(t-1)}.
\end{align}
We additionally define $B^{(0)}=1$.
By the definition of $B^{(t)}$ and that $\prog(v^{(t)}_{\mathrm{w}_i\rightarrow\mathrm{s}})\leq \prog(v^{(t,\star)}_{\mathrm{w}_i\rightarrow\mathrm{s}})$ by \eqref{eqn:fiqpjfqcv}, it naturally holds that 
\begin{align}
    B^{(t-1)}\leq B^{(t)}&=\max_{1\leq s\leq t}\max_{1\leq i\leq n}\prog(v^{(s)}_{\mathrm{w}_i\rightarrow\mathrm{s}})+1 \nonumber \\
    &=\max\left\{B^{(t-1)}, \max_{1\leq i\leq n}\prog(v^{(t)}_{\mathrm{w}_i\rightarrow\mathrm{s}})+1\right\}\overset{\eqref{eqn:bvjowenf2q-dfhiafsa}}{\leq}B^{(t-1)}+1.\label{eqn:ghienfgqgeq}
\end{align}
Therefore, one round of communication can increase $B^{(t)}$ at most by $1$.

From \eqref{eqn:ghienfgqgeq}, we find that {$B^{(t)}=B^{(t-1)}+1$}  only if $\max_{1\leq i\leq n}\prog(v^{(t,\star)}_{\mathrm{w}_i\rightarrow\mathrm{s}})=\max_{1\leq i\leq n}\prog(v^{(t)}_{\mathrm{w}_i\rightarrow\mathrm{s}})$. Let $k=\max_{1\leq i\leq n}\prog(v^{(t,\star)}_{\mathrm{w}_i\rightarrow\mathrm{s}})$. Recall that the compressors constructed in Example 2 are built on shared randomness, we therefore conclude that $\max_{1\leq i\leq n}\prog(v^{(t)}_{\mathrm{w}_i\rightarrow\mathrm{s}})=\max_{1\leq i\leq n}\prog(v^{(t,\star)}_{\mathrm{w}_i\rightarrow\mathrm{s}})=k$ is equivalent to that coordinate index $k$ is chosen to communicate in communication round $t$, 
% this round of communication, 
which is of probability $\frac{s}{d}$. Therefore, we have 
\begin{align*}
    \PP(B^{(t)}=B^{(t-1)}+1)\leq& \PP(\max_{1\leq i\leq n}\prog(v^{(t)}_{\mathrm{w}_i\rightarrow\mathrm{s}})=\max_{1\leq i\leq n}\prog(v^{(t,\star)}_{\mathrm{w}_i\rightarrow\mathrm{s}}))\nonumber\\
    =&\PP\left(\text{the coordinate index }\max_{1\leq i\leq n}\prog(v^{(t,\star)}_{\mathrm{w}_i\rightarrow\mathrm{s}})\text{ is chosen at round $t$}\right)=\frac{s}{d}.
\end{align*}
Let $E^{(t)}$ be the event $\{\text{the coordinate index }\max_{1\leq i\leq n}\prog(v^{(t,\star)}_{\mathrm{w}_i\rightarrow\mathrm{s}})\text{ is chosen to communicate at round $t$}\}$. Since the compression is uniformly random, we have $\mathds{{1}}(E^{(1)}),\dots,\mathds{{1}}(E^{(T)})\overset{\text{i.i.d.}}{\sim}\mathrm{Bernoulli}(\frac{s}{d})$ where $\mathds{1}(\cdot)$ is the indicator function. By the above argument, we also have $B^{(t)}-B^{(t-1)}\leq \mathds{1}(E^{(t)})$ for any $1\leq t\leq T$.
As a result, we reach
\begin{align}
    \PP(B^{(T)}\geq d) \leq& e^{-d}\EE[e^{B^{(T)}}]=e^{-d}\EE\left[\exp\left(B^{(0)}+\sum_{t=1}^{T}(B^{(t)}-B^{(t-1)})\right)\right]\nonumber\\
    \leq& e^{-(d-1)}\EE\left[\exp\left(\sum_{t=1}^{T}\mathds{1}(E^{(t)})\right)\right]=e^{-(d-1)}\prod_{t=1}^{T}\EE\left[\exp\left(\mathds{1}(E^{(t)})\right)\right]\nonumber\\
    =&e^{-(d-1)}\prod_{t=1}^{T}\left(1+\frac{s}{d}(e-1)\right)\leq e^{-(d-1)}\prod_{t=1}^{T}e^{(e-1)s/d}=e^{(e-1)Ts/d-d+1},\nonumber
\end{align}
which leads to the conclusion.
\end{proof}

\subsection{Proof of Theorem \ref{thm:lower3}}\label{app:proof-lb-2}
Theorem \ref{thm:lower3} essentially follows the same analysis as in Theorem \ref{thm:lower1}. The only difference is that we shall construct compressors in proving $\Omega(\frac{\Delta L}{\delta T})$ by using the rand-$s$ operators with shared randomness and $s=\lceil \delta d\rceil$. There is no scaling procedure in compression.
One can easily verify that
\begin{align*}
   \EE[\|C_i(x)-x\|^2]=&\sum_{k=1}^d\EE\left[\left(x_k\mathds{1}\{k \text{ is chosen}\}-x_k\right)^2\right]\\
   =&\sum_{k=1}^dx_k^2\PP(k \text{ is not chosen})=\left(1-\frac{s}{d}\right)\|x\|^2\leq \delta\|x\|^2,
\end{align*}
where the last inequality follows the definition of $s$. 
Therefore, we have $\{C_i\}_{i=1}^n\subseteq\cC_\delta$.
The scaling procedure does not change $\prog$, and thus 
has no effect on the argument in terms of non-zero coordinates. By considering $\omega = \delta^{-1}-1$, \emph{\ie}, $1+\omega =\delta$, we can easily adapt the proof of Theorem \ref{thm:lower1} to reach Theorem \ref{thm:lower3}.

\section{Convergence of NEOLITHIC}\label{app:convergence}
\subsection{Proof of Lemma \ref{lem:fcc-expo}}\label{app:fcc}
Let $v^{(k,r)}$ be the $r$-th intermediate variable generated in FCC (\ie, Algorithm \ref{alg:fcc}) for any $0\leq r\leq R$. Since $C$ is $\delta$-contractive, we have
\begin{align*}
    \EE[\|v^{(k,R)}-v^{(k,\star)}\|^2]&=\EE[\|v^{(k,R-1)}+C(v^{(k,\star)}-v^{(k,R-1)} )-v^{(k,\star)}\|^2]\\
    &=\EE\left[\EE[\|C(v^{(k,\star)}-v^{(k,R-1)} )+v^{(k,R-1)}-v^{(k,\star)}\|^2\mid v^{(k,R-1)}]\right]\\
    &\leq (1-\delta)\EE[\|v^{(k,R-1)}-v^{(k,\star)}\|^2].
\end{align*}
Iterating the above inequality with respect to $r= R-1,\dots,1, 0$, we reach 
\begin{align*}
     &\EE[\|v^{(k,R)}-v^{(k,\star)}\|^2]\leq (1-\delta)\EE[\|v^{(k,R-1)}-v^{(k,\star)}\|^2] \\
     \leq& (1-\delta)^2\EE[\|v^{(k,R-2)}-v^{(k,\star)}\|^2]\leq \cdots\leq  (1-\delta)^R\EE[\|v^{(k,0)}-v^{(k,\star)}\|^2]=(1-\delta)^R\|v^{(k,\star)}\|^2.
\end{align*}

\subsection{Proof of Theorem \ref{thm:convergence-rate}}\label{app:convergence-rate}
In this subsection, we provide the convergence proof for NEOLITHIC with bidirectional compression and contractive compressors.
We first introduce some notations:   $\Omega^{(k)}:=\er^{(k)}+\frac{1}{n}\sum_{i=1}^n\er_i^{(k)}$, $\forall\,k\geq 0$; $y^{(k)}:=x^{(k)}-\gamma \Omega^{(k)}$, $\forall\,k\geq 0$; $\Psi^{(k)}:=\|\er^{(k)}\|^2+\frac{3}{n}\sum_{i=1}^n\|\er_i^{(k)}\|^2$, $\forall\,k\geq 0$. We will use the following lemmas.

\begin{lemma}[\sc Recursion Formula]\label{lem:reccur}
For any $k\geq 0$, it holds that 
\begin{align*}
     x^{(k+1)}-x^{(k)}=-\frac{\gamma}{n}\sum_{i=1}^n\hat{g}_i^{(k)}-\gamma\Omega^{(k)}+\gamma\Omega^{(k+1)}.
\end{align*}
\end{lemma}
\begin{proof}
It is observed from Algorithm \ref{alg:fcc} that the vector $v^{(k,R)}$ returned by the FCC operator satisfies 
\begin{align}\label{q87zd}
 v^{(k,R)} = \sum_{r=0}^{R-1} c^{(k,r)}  \quad \mbox{(or }  v_i^{(k,R)} = \sum_{r=0}^{R-1} c_i^{(k,r)} \mbox{ if FCC is utilized in worker $i$)}
\end{align}
With \eqref{q87zd}, the relation between $\tilde{g}^{(k)}$ and $\er^{(k)}$ (or between  $\tilde{g}_i^{(k)}$ and $\er_i^{(k)}$) in Algorithm \ref{alg:mcdc} satisfies
% By definition of the FCC module (see Algorithm \ref{alg:fcc}), we have that 
\begin{align}
   \tilde{g}^{(k)}-\er^{(k+1)}= \sum_{r=0}^{R-1}c^{(k,r)}\quad \text{and}\quad \tilde{g}_i^{(k)}-\er_i^{(k+1)}=\sum_{r=0}^{R-1}c_i^{(k,r)},\quad \forall\,1\leq i\leq n.\label{eqn:vjoenoq}
\end{align}
Therefore, we have for any $k\geq 0$
\begin{align}
    x^{(k+1)}-x^{(k)}&=-\gamma \sum_{r=0}^{R-1}c^{(k,r)}\overset{\eqref{eqn:vjoenoq}}{=}-\gamma(\tilde{g}^{(k)}-\er^{(k+1)})\nonumber\\
    &=-\gamma\left(\er^{(k)}+\frac{1}{n}\sum_{i=1}^n\sum_{r=0}^{R-1}c_i^{(k,r)}-\er^{(k+1)}\right)\label{eqn:tmepm1}\\
    &\overset{\eqref{eqn:vjoenoq}}{=}-\gamma\left(\er^{(k)}+\frac{1}{n}\sum_{i=1}^n\left(\tilde{g}_i^{(k)}-\er_i^{(k+1)}\right)-\er^{(k+1)}\right)\nonumber\\
    &=-\gamma\left(\er^{(k)}+\frac{1}{n}\sum_{i=1}^n\left(\hat{g}_i^{(k)}+\er_i^{(k)}-\er_i^{(k+1)}\right)-\er^{(k+1)}\right)\label{eqn:tmepm2}\\
    &=-\frac{\gamma}{n}\sum_{i=1}^n\hat{g}_i^{(k)}-\gamma\Omega^{(k)}+\gamma\Omega^{(k+1)}\label{eqn:tmepm3}
\end{align}
where \eqref{eqn:tmepm1} and \eqref{eqn:tmepm2} follow the implementation of Algorithm \ref{alg:mcdc}, and we use the notation of $\Omega^{(k)}$ in \eqref{eqn:tmepm3}.
\end{proof}

\begin{lemma}[\sc Descent Lemma]\label{lem:descent}
Let the auxiliary sequence be $y^{(k)}:=x^{(k)}-\gamma \Omega^{(k)}$, $\forall\,k\geq0$. Under Assumption 
\ref{asp:smooth}, if learning rate $0<\gamma\leq \frac{1}{2L}$, it holds that for any $k\geq 0$,
\begin{align}
    \EE[f(y^{(k+1)})]\leq \EE[f(y^{(k)})]-\frac{\gamma}{4}\EE[\|\nabla f(x^{(k)})\|^2]+2\gamma^3 L^2 \EE[\|\Omega^{(k)}\|^2] +\frac{\gamma^2 L \sigma^2}{2nR}.\label{eqn:desecnt}
\end{align}
\end{lemma}
\begin{proof}
By Lemma \ref{lem:reccur} and the definition of $y^{(k)}$ for $k\geq 0$, we directly have that 
\begin{align*}
    y^{(k+1)}=y^{(k)}-\frac{\gamma}{n}\sum_{i=1}^n\hat{g}_i^{(k)}.
\end{align*}
Since $f$ is $L$-smooth, we have 
\begin{align}
    f(y^{(k+1)})&\leq f(y^{(k)})+\langle \nabla f(y^{(k)}), y^{(k+1)}-y^{(k)}\rangle +\frac{L}{2}\|y^{(k+1)}-y^{(k)}\|^2\nonumber\\
    &=f(y^{(k)})-\gamma \left\langle \nabla f(y^{(k)}), \frac{1}{n}\sum_{i=1}^n\hat{g}_i^{(k)}\right\rangle +\frac{\gamma^2 L}{2}\left\|\frac{1}{n}\sum_{i=1}^n\hat{g}_i^{(k)}\right\|^2.\label{eqn:vjoqwnmoqv}
\end{align}
Since $\hat{g}_i^{(k)}=\frac{1}{R}\sum_{r=1}^RO_i(x^{(k)};\zeta_i^{(k,r)})$ is a unbiased estimator of $\nabla f_i(x^{(k)})$, and by Assumption \ref{asp:gd-noise}, we have 
\begin{align}\label{eqn:vjoqwenq1}
    \EE\left[\frac{1}{n}\sum_{i=1}^n\hat{g}_i^{(k)}\right]=\frac{1}{n}\sum_{i=1}^n\nabla f_i(x^{(k)})=\nabla f(x^{(k)})
\end{align}
and 
\begin{align}
    \EE\left[\left\|\frac{1}{n}\sum_{i=1}^n\hat{g}_i^{(k)}\right\|^2\right]&=\|\nabla f(x^{(k)})\|^2+\EE\left[\left\|\frac{1}{n}\sum_{i=1}^n\hat{g}_i^{(k)}-\nabla f(x^{(k)})\right\|^2\right]\nonumber\\
    &=\|\nabla f(x^{(k)})\|^2+\EE\left[\left\|\frac{1}{nR}\sum_{i=1}^n\sum_{r=0}^{R-1}\left(O_i(x^{(k)};\zeta_i^{(k,r)})-\nabla f_i(x^{(k)})\right)\right\|^2\right]\nonumber\\
    &\leq \|\nabla f(x^{(k)})\|^2+ \frac{\sigma^2}{nR}.\label{eqn:vjoqwenq2}
\end{align}
Taking global expectation over \eqref{eqn:vjoqwnmoqv}, and plugging \eqref{eqn:vjoqwenq1} and \eqref{eqn:vjoqwenq2} into it, we reach
\begin{align}
    &\EE[f(y^{(k+1)})]-\EE[f(y^{(k)})]\nonumber\\
    \leq& -\gamma \EE[\langle \nabla f(y^{(k)}), \nabla f(x^{(k)})\rangle] +\frac{\gamma^2 L}{2}\EE[\|\nabla f(x^{(k)})\|^2]+\frac{\gamma^2 L\sigma^2}{2nR}\nonumber\\
    =&-\gamma \EE[\langle \nabla f(y^{(k)})-\nabla f(x^{(k)}), \nabla f(x^{(k)})\rangle] -\left(\gamma-\frac{\gamma^2 L}{2}\right)\EE[\|\nabla f(x^{(k)})\|^2]+\frac{\gamma^2 L\sigma^2}{2nR}\nonumber\\
    \leq &2\gamma  \EE[\|\nabla f(y^{(k)})-\nabla f(x^{(k)})\|^2] +\frac{\gamma}{2}\EE[\|\nabla f(x^{(k)})\|^2]\nonumber \\
    &\quad -\left(\gamma-\frac{\gamma^2 L}{2}\right)\EE[\|\nabla f(x^{(k)})\|^2]+\frac{\gamma^2 L\sigma^2}{2nR}\label{eqn:gjpo1ofq1}\\
    \leq &2\gamma L^2 \EE[\|y^{(k)}-x^{(k)}\|^2]  -\frac{\gamma(1-\gamma L)}{2}\EE[\|\nabla f(x^{(k)})\|^2]+\frac{\gamma^2 L\sigma^2}{2nR},\label{eqn:gjpo1ofq2}
\end{align}
where we use Young's inequality in \eqref{eqn:gjpo1ofq1}, and \eqref{eqn:gjpo1ofq2} holds by Assumption \ref{asp:smooth}. Using $y^{(k)}-x^{(k)}=-\gamma\Omega^{(k)}$ and that
\begin{align*}
    0< \gamma\leq \frac{1}{2L}\quad \Longrightarrow \quad \frac{\gamma(1-\gamma L)}{2}\geq \frac{\gamma}{4}
\end{align*} in \eqref{eqn:gjpo1ofq2}, we reach the conclusion in this lemma.
\end{proof}

\begin{lemma}[\sc Vanishing Error]\label{lem:vanish-error}
Assume $R$ is sufficiently large so that $(1-\delta)^R< \frac{1}{4}$ and we let $\Psi^{(k)}:=\|\er^{(k)}\|^2+\frac{3}{n}\sum_{i=1}^n\|\er_i^{(k)}\|^2$ for $k\geq 0$. Under Assumptions \ref{asp:gd-noise}, \ref{ass:contract}, \ref{asp:gd-hetero},  it holds that 
\begin{align*}
    \EE[\Psi^{(k+1)}]\leq 4(1-\delta)^R\left(\EE[\Psi^{(k)}]+5\EE[\|\nabla f(x^{(k)})\|^2]+4\left( b^2+\frac{\sigma^2}{R}\right)\right).
\end{align*}
\end{lemma}
\begin{proof}
By Lemma \ref{lem:fcc-expo}, it holds that
\begin{align*}
    \EE[\|\er^{(k)}\|^2]&=\EE[\|\tilde{g}^{(k)}-\text{FCC}(\tilde{g}^{(k)},C,R)\|^2]\leq (1-\delta)^R\EE[\|\tilde{g}^{(k)}\|^2].
\end{align*}
Note that
\begin{align*}
    \tilde{g}^{(k)}&=\er^{(k)}+\frac{1}{n}\sum_{i=1}^n\sum_{r=0}^{R-1}c_i^{(k,r)}=\er^{(k)}+\frac{1}{n}\sum_{i=1}^n\left(\tilde{g}_i^{(k)}-\er_i^{(k+1)}\right)\\
    &=\er^{(k)}+\frac{1}{n}\sum_{i=1}^n\left(\hat{g}_i^{(k)}+\er_i^{(k)}-\er_i^{(k+1)}\right)\\
     &=\er^{(k)}+\frac{1}{n}\sum_{i=1}^n\hat{g}_i^{(k)}+\frac{1}{n}\sum_{i=1}^n\er_i^{(k)}-\frac{1}{n}\sum_{i=1}^n\er_i^{(k+1)}.
\end{align*}
Therefore, by using Young's inequality and \eqref{eqn:vjoqwenq2}, we have
\begin{align}
    \EE[\|\er^{(k)}\|^2]&\leq (1-\delta)^R\EE[\|\tilde{g}^{(k)}\|^2]\nonumber\\
    &\leq (1-\delta)^R\Biggl(4\EE[\|\er^{(k)}\|^2]+4\EE\left[\left\|\frac{1}{n}\sum_{i=1}^n\hat{g}_i^{(k)}\right\|^2\right]\nonumber\\
    &\qquad\qquad\qquad  +\frac{4}{n}\sum_{i=1}^n\EE[\|\er_i^{(k)}\|^2]+\frac{4}{n}\sum_{i=1}^n\EE[\|\er_i^{(k+1)}\|^2]\Biggl)\nonumber\\
    &\leq (1-\delta)^R\Biggl(4\EE[\|\er^{(k)}\|^2]+4\EE[\|\nabla f(x^{(k)})\|^2]\nonumber\\
    &\qquad\qquad\qquad   +\frac{4}{n}\sum_{i=1}^n\EE[\|\er_i^{(k)}\|^2]+\frac{4}{n}\sum_{i=1}^n\EE[\|\er_i^{(k+1)}\|^2]+\frac{4\sigma^2}{nR}\Biggl).\label{eqn:fsadLvjqon1}
\end{align}
For any $1\leq i\leq n$, using the similar argument, we have that 
\begin{align}
    \EE[\|\er_i^{(k+1)}\|^2]=&\EE[\|\tilde{g}^{(k)}-\text{FCC}(\tilde{g}_i^{(k)},C_i,R)\|^2]\leq (1-\delta)^R\EE[\|\tilde{g}_i^{(k)}\|^2]\nonumber\\
    =& (1-\delta)^R\EE[\|\hat{g}_i^{(k)}+\er_i^{(k)}\|^2]\nonumber\\
    \leq &(1-\delta)^R\left(2\EE[\|\hat{g}_i^{(k)}\|^2]+2\EE[\|\er_i^{(k)}\|^2]\right)\nonumber
\\
    = &(1-\delta)^R\left(2\EE[\|\nabla f_i(x^{(k)})\|^2]+2\EE[\|\er_i^{(k)}\|^2]+\frac{2\sigma^2}{R}\right).\label{eqn:vjoqwnqfv}
\end{align}
Note that  for any $1\leq i\leq n$,
\begin{align}
    \EE[\|\nabla f_i(x^{(k)})\|^2]&=\EE[\|\nabla f_i(x^{(k)})-\nabla f(x^{(k)})+\nabla f(x^{(k)})\|^2]\nonumber\\
    &\leq 2\EE[\|\nabla f_i(x^{(k)})-\nabla f(x^{(k)})\|^2]+2\EE[\|\nabla f(x^{(k)})\|^2]\label{eqn:vhjoqwnfq}.
\end{align}
Taking the average of \eqref{eqn:vjoqwnqfv} over all $1\leq i\leq n$, and using \eqref{eqn:vhjoqwnfq} and Assumption \ref{asp:gd-hetero}, 
we further have that
\begin{align}
    &\frac{1}{n}\sum_{i=1}^n\EE[\|\er_i^{(k+1)}\|^2]\nonumber\\
    =& (1-\delta)^R\Biggl(\frac{4}{n}\sum_{i=1}^n\EE[\|\nabla f_i(x^{(k)})-\nabla f(x^{(k)})\|^2]+4\EE[\|\nabla f(x^{(k)})\|^2]\nonumber\\
    &\qquad\qquad +\frac{2}{n}\sum_{i=1}^n\EE[\|\er_i^{(k)}\|^2]+\frac{2\sigma^2}{R}\Biggl)\nonumber\\
    \leq &(1-\delta)^R\left(4\EE[\|\nabla f(x^{(k)})\|^2]+\frac{2}{n}\sum_{i=1}^n\EE[\|\er_i^{(k)}\|^2]+4b^2+\frac{2\sigma^2}{R}\right).\label{eqn:fsadLvjqon2}
\end{align}

With \eqref{eqn:fsadLvjqon1} and \eqref{eqn:fsadLvjqon2}, we obtain
\begin{align*}
    & \EE[\|\er^{(k+1)}\|^2]+ \frac{4}{n}\sum_{i=1}^n\EE[\|\er_i^{(k+1)}\|^2]\\
    \leq &(1-\delta)^R\Biggl(4\EE[\|\er^{(k)}\|^2]+4\EE[\|\nabla f(x^{(k)})\|^2]+\frac{4}{n}\sum_{i=1}^n\EE[\|\er_i^{(k)}\|^2]+\frac{4}{n}\sum_{i=1}^n\EE[\|\er_i^{(k+1)}\|^2]\nonumber\\
    &\qquad\qquad\quad   +\frac{4\sigma^2}{nR}+16\EE[\|\nabla f(x^{(k)})\|^2]+\frac{8}{n}\sum_{i=1}^n\EE[\|\er_i^{(k)}\|^2]+16 b^2+\frac{8\sigma^2}{R}\Biggl)\\
    \leq &(1-\delta)^R\Biggl(4\EE[\|\er^{(k)}\|^2]+ \frac{12}{n}\sum_{i=1}^n\EE[\|\er_i^{(k)}\|^2]+\frac{4}{n}\sum_{i=1}^n\EE[\|\er_i^{(k+1)}\|^2]\\
    &\qquad\qquad\quad  +20\EE[\|\nabla f(x^{(k)})\|^2]+16 b^2+\frac{12\sigma^2}{R}\Biggl).
\end{align*}
We thus have 
\begin{align*}
    & \EE[\|\er^{(k+1)}\|^2]+ \frac{4(1-(1-\delta)^R)}{n}\sum_{i=1}^n\EE[\|\er_i^{(k+1)}\|^2]\\
    \leq &(1-\delta)^R\Biggl(4\EE[\|\er^{(k)}\|^2]+ \frac{12}{n}\sum_{i=1}^n\EE[\|\er_i^{(k)}\|^2]+20\EE[\|\nabla f(x^{(k)})\|^2]+16 b^2+\frac{12\sigma^2}{R}\Biggl).
\end{align*}
Since $R$ is sufficiently large such that \begin{equation}
    (1-\delta)^R< \frac{1}{4},
\end{equation}
we have that $4(1-(1-\delta)^R)\geq 3$ and hence \begin{align*}
    \EE[\Psi^{(k+1)}]&\leq \EE[\|\er^{(k+1)}\|^2]+ \frac{4(1-(1-\delta)^R)}{n}\sum_{i=1}^n\EE[\|\er_i^{(k+1)}\|^2]\\
    &\leq 4(1-\delta)^R\left(\EE[\Psi^{(k)}]+5\EE[\|\nabla f(x^{(k)})\|^2]+4 \left(b^2+\frac{\sigma^2}{R}\right)\right).
\end{align*}
\end{proof}
With Lemma \ref{lem:vanish-error}, we easily reach its ergodic version:
\begin{lemma}[\sc Ergodic Vanishing Error]\label{lem:ergodic-error}
We let $R$ sufficiently large so that $\theta \triangleq 4(1-\delta)^R<1$.  Under Assumptions \ref{asp:gd-noise}, \ref{ass:contract}, \ref{asp:gd-hetero}, it holds for any $K\geq0$ that,
\begin{align*}
    \frac{1}{K+1}\sum_{k=0}^{K} \EE[\Psi^{(k)}]\leq \frac{5\theta}{(K+1)(1-\theta)}\sum_{k=0}^{K-1}\EE[\|\nabla f(x^{(k)})\|^2]+ \frac{4\theta}{1-\theta} \left( b^2+\frac{\sigma^2}{R}\right).
\end{align*}
\end{lemma}
\begin{proof}
Let $\theta =4(1-\delta)^R$, then by Lemma \ref{lem:vanish-error} and noting $\Psi^{(0)}=0$, we have 
\begin{align}
    \EE[\Psi^{(k+1)}]\leq& \theta \left(\EE[\Psi^{(k)}]+5 \EE[\|\nabla f(x^{(k)})\|^2]+4\left( b^2+\frac{\sigma^2}{R}\right)\right)\nonumber\\
    \leq &\theta \Biggl(\theta \left(\EE[\Psi^{(k-1)}]+5 \EE[\|\nabla f(x^{(k-1)})\|^2]+4\left( b^2+\frac{\sigma^2}{R}\right)\right)\nonumber\\
    &\quad +5 \EE[\|\nabla f(x^{(k)})\|^2]+4\left( b^2+\frac{\sigma^2}{R}\right)\Biggl)\nonumber\\
    = &\theta^2 \EE[\Psi^{(k-1)}]+5\sum_{\ell=k-1}^k \theta^{k+1-\ell}  \EE[\|\nabla f(x^{(\ell)})\|^2]+4 \sum_{\ell=k-1}^k \theta^{k+1-\ell}\left( b^2+\frac{\sigma^2}{R}\right)\nonumber\\
    \leq &\cdots\nonumber\\
    \leq &\theta^{k+1} \EE[\Psi^{(0)}]+5\sum_{\ell=0}^k \theta^{k+1-\ell}  \EE[\|\nabla f(x^{(\ell)})\|^2]+4 \sum_{\ell=0}^k \theta^{k+1-\ell}\left( b^2+\frac{\sigma^2}{R}\right)\nonumber\\
    =&5\sum_{\ell=0}^k \theta^{k+1-\ell}  \EE[\|\nabla f(x^{(\ell)})\|^2]+4 \sum_{\ell=0}^k \theta^{k+1-\ell}\left( b^2+\frac{\sigma^2}{R}\right)\label{eqn:vj0ownof1}.
\end{align}
Therefore, by taking the average of \eqref{eqn:vj0ownof1} over $k=0,\dots,K-1$ and using $\Psi^{(0)}=0$, we further have
\begin{align*}
    &\frac{1}{K+1}\sum_{k=0}^{K} \EE[\Psi^{(k)}]=\frac{1}{K+1}\sum_{k=0}^{K-1} \EE[\Psi^{(k+1)}]\\
    \leq& \frac{1}{K+1}\sum_{k=0}^{K-1}\left(5\sum_{\ell=0}^k \theta^{k+1-\ell}  \EE[\|\nabla f(x^{(\ell)})\|^2]+4 \sum_{\ell=0}^k \theta^{k+1-\ell}\left( b^2+\frac{\sigma^2}{R}\right)\right)\\
    =& \frac{1}{K+1}\left(5\sum_{\ell=0}^{K-1}\EE[\|\nabla f(x^{(\ell)})\|^2]\left(\sum_{k=\ell}^{K-1} \theta^{k+1-\ell}\right)  +4\left( b^2+\frac{\sigma^2}{R}\right)\sum_{\ell=0}^{K-1} \sum_{k=\ell}^{K-1} \theta^{k+1-\ell}\right),
\end{align*}
where we change the summation order of indexes $k$ and $\ell$ in the last identity. Since $\sum_{k=\ell}^{K-1} \theta^{k+1-\ell}=\frac{\theta(1-\theta^{K-\ell})}{1-\theta}\leq \frac{\theta}{1-\theta}$, we thus have
\begin{align*}
    &\frac{1}{K+1}\sum_{k=0}^{K} \EE[\Psi^{(k)}]\\
    \leq &\frac{1}{K+1}\left(5\sum_{\ell=0}^{K-1}\EE[\|\nabla f(x^{(\ell)})\|^2]\frac{\theta}{1-\theta}  +4\left( b^2+\frac{\sigma^2}{R}\right)\frac{K\theta}{1-\theta} \right)\\
    \leq &\frac{1}{K+1}\frac{5\theta}{1-\theta}\sum_{\ell=0}^{K-1}\EE[\|\nabla f(x^{(\ell)})\|^2]+ \frac{4\theta}{1-\theta} \left( b^2+\frac{\sigma^2}{R}\right).
\end{align*}
\end{proof}

Given the above lemmas, now we prove the convergence rate of NEOLITHIC.
\begin{theorem}
Let the communication round be $R=\left\lceil{\max\left\{\ln\left({\delta T\max\{b^2,\sigma^2\delta \}}/{\Delta L}\right),\ln(8)\right\}}/{\delta}\right\rceil$ and learning rate as in \eqref{eqn:learning-rate}.  Under Assumptions \ref{asp:smooth}, \ref{asp:gd-noise}, \ref{ass:contract}, \ref{asp:gd-hetero}, it holds that for any $K\geq 0$, 
\begin{align*}
    \frac{1}{K+1}\sum_{k=0}^K \EE[\|\nabla f(x^{(k)})\|^2] =  \tilde{O}\Big(\Big(\frac{\Delta L \sigma^2}{nT}\Big)^\frac{1}{2}+
\frac{\Delta L}{\delta T}\Big),
\end{align*}
where $T=KR$ is the total number of gradient queries (or rounds of compressed communications) on each worker.
\end{theorem}
\begin{proof}
Averaging \eqref{eqn:desecnt} over $k=0,\dots,K$, and using the fact that $y^{(0)}=x^{(0)}$ and $f(y^{(K+1)})\geq f^\star$, we have
\begin{align}
    &\frac{1}{K+1}\sum_{k=0}^K\EE[\|\nabla f(x^{(k)})\|^2]\nonumber\\
    \leq&\frac{ 4\EE[f(y^{(0)})]-4\EE[f(y^{(K+1)})]}{\gamma(K+1)}+\frac{8\gamma^2 L^2}{K+1}\sum_{k=0}^K \EE[\|\Omega^{(k)}\|^2] +\frac{2\gamma L \sigma^2}{nR}\nonumber\\
    \leq &\frac{4(f(x^{(0)})-f^\star)}{\gamma(K+1)}+\frac{8\gamma^2 L^2}{K+1}\sum_{k=0}^K \EE[\|\Omega^{(k)}\|^2] +\frac{2\gamma L \sigma^2}{nR}.\label{eqn:jvoqwenw21}
\end{align}
By the definition of $\Omega^{(k)}$ and using the Young's inequality, it holds that
\begin{align*}
    \EE[\|\Omega^{(k)}\|^2]=&\EE\left[\left\|\er^{(k)}+\frac{1}{n}\sum_{i=1}^n\er_i^{(k)}\right\|^2\right]\\
    \leq&\left(1+\frac{1}{3}\right)\EE[\|\er^{(k)}\|^2]+(1+3)\EE\left[\left\|\frac{1}{n}\sum_{i=1}^n\er_i^{(k)}\right\|^2\right]\\
    \leq&\frac{4}{3}\EE[\|\er^{(k)}\|^2]+\frac{4}{n}\sum_{i=1}^n\EE[\|\er_i^{(k)}\|^2]=\frac{4}{3}\EE[\Psi^{(k)}].
\end{align*}
Therefore, by Lemma \ref{lem:ergodic-error}, we reach
\begin{align}
    &\frac{1}{K+1}\sum_{k=0}^K \EE[\|\Omega^{(k)}\|^2]\leq \frac{4}{3(K+1)}\sum_{k=0}^K \EE[\|\Psi^{(k)}\|^2]\nonumber\\
    \leq& \frac{20\theta}{3(K+1)(1-\theta)}\sum_{k=0}^{K-1}\EE[\|\nabla f(x^{(k)})\|^2]+ \frac{16\theta}{3(1-\theta)} \left( b^2+\frac{\sigma^2}{R}\right).\label{eqn:bjoeqnf}
\end{align}
Plugging \eqref{eqn:bjoeqnf} into \eqref{eqn:jvoqwenw21}, we reach that 
\begin{align}
    &\frac{1}{K+1}\sum_{k=0}^K\EE[\|\nabla f(x^{(k)})\|^2]\nonumber\\
    \leq &\frac{4(f(x^{(0)})-f^\star)}{\gamma(K+1)} +
    \frac{160\gamma^2L^2\theta}{3(1-\theta)(K+1)}\sum_{k=0}^{K-1}\EE[\|\nabla f(x^{(k)})\|^2]
    \nonumber\\
    &\quad + \frac{128\gamma^2L^2 \theta}{3(1-\theta)} \left( b^2+\frac{\sigma^2}{R}\right)+\frac{2\gamma L \sigma^2}{nR}.\label{eqn:gv0nefsvbwe}
\end{align}
Assume the learning rate $\gamma$ is sufficiently small such that 
\begin{align}\label{eqn:jgowemgwe}
    \gamma\leq \frac{1}{11L}\sqrt{\frac{1-\theta}{\theta}}\quad \Longrightarrow\quad \frac{160\gamma^2L^2\theta}{3(1-\theta)}\leq \frac{1}{2}.
\end{align}
then we can further bound \eqref{eqn:gv0nefsvbwe} as 
\begin{align}
    \frac{1}{K+1}\sum_{k=0}^K\EE[\|\nabla f(x^{(k)})\|^2]\leq \frac{8\Delta}{\gamma(K+1)} +\frac{4\gamma L \sigma^2}{nR}+ \frac{256\gamma^2L^2 \theta}{3(1-\theta)} \left( b^2+\frac{\sigma^2}{R}\right).\label{xncxn-cn-2-test}
\end{align}
where we use  $\Delta \geq f(x^{(0)})-f^\star$.
By choosing 
\begin{equation}\label{eqn:learning-rate}
    \gamma = \frac{1}{11L+\sigma\left(\frac{(K+1)L}{2nR\Delta}\right)^\frac{1}{2}+\left(\frac{32(K+1)L^2\theta (b^2+\sigma^2/R) }{3(1-\theta)\Delta}\right)^\frac{1}{3}},
\end{equation}
we have 
\begin{align}
&\frac{1}{K+1}\sum_{k=0}^K \EE[\|\nabla f(x^{(k)})\|^2] \nonumber\\
\le& \frac{16\sigma  \sqrt{L\Delta}}{\sqrt{nR(K+1)}}+\frac{36\Delta^\frac{2}{3}L^\frac{2}{3}\theta^\frac{1}{3}(b^2+\sigma^2/R)^\frac{1}{3}}{(K+1)^\frac{2}{3}(1-\theta)^\frac{1}{3}}+\frac{88L\Delta}{K+1}\nonumber\\
=&O\Big(\Big(\frac{\Delta L \sigma^2}{nT}\Big)^\frac{1}{2}+\frac{\theta^{\frac{1}{3}}\Delta^{\frac{2}{3}}L^\frac{2}{3}\max\{b^{\frac{2}{3}},\sigma^{\frac{2}{3}}/R^\frac{1}{3}\}}{(1-\theta)^{\frac{1}{3}}K^{\frac{2}{3}}} +\frac{\Delta L}{K}\Big).\label{eqn:xbsd87}
\end{align}
Let 
\begin{align}\label{eqn:R-1}
    R =\Bigg\lceil \frac{\max\{\ln\left(\frac{\delta T\max\{b^2,\sigma^2\delta \}}{\Delta L}\right),\ln(8)\}}{\delta}\Bigg\rceil=\tilde{O}\left(\frac{1}{\delta}\right),
\end{align}
then it holds that 
\begin{align}
    \theta=&4(1-\delta)^R\leq 4e^{-\delta R}\nonumber\\
    \leq& 4\exp\left(-{\max\{\ln({\delta T\max\{b^2,\sigma^2\delta \}}/{\Delta L}),\ln(8)\}}\right)=\min\left\{\frac{4\Delta L}{\delta T\max\{b^2,\sigma^2\delta \}},\frac{1}{2}\right\},\label{eqn:R-2}
\end{align}
and hence $1-\theta = \Omega(1)$ and $\gamma$ satisfies \eqref{eqn:jgowemgwe}.
Plugging \eqref{eqn:R-1} and \eqref{eqn:R-2} into \eqref{eqn:xbsd87}, and applying the notation $T=KR$, we reach
\begin{align}
\frac{1}{K+1}\sum_{k=0}^K \EE[\|\nabla f(x^{(k)})\|^2] = &O\Big( \Big(\frac{\Delta L \sigma^2}{nT}\Big)^\frac{1}{2}+ 
\frac{\theta^{\frac{1}{3}}R^\frac{2}{3}\Delta^{\frac{2}{3}}L^\frac{2}{3}\max\{b^{\frac{2}{3}},\sigma^{\frac{2}{3}}/R^\frac{1}{3}\}}{T^{\frac{2}{3}}} + \frac{R\Delta L}{T}\Big)\nonumber\\
=&\tilde{O}\Big( \Big(\frac{\Delta L \sigma^2}{nT}\Big)^\frac{1}{2} +
\frac{\theta^{\frac{1}{3}}\Delta^{\frac{2}{3}}L^\frac{2}{3}\max\{b^{\frac{2}{3}},\sigma^{\frac{2}{3}}\delta^\frac{1}{3}\}}{\delta^{\frac{2}{3}}T^{\frac{2}{3}}} + \frac{\Delta L}{\delta T}\Big)\nonumber\\
=& \tilde{O}\Big(\Big(\frac{\Delta L \sigma^2}{nT}\Big)^\frac{1}{2}+ 
\frac{\Delta L}{\delta T}\Big)\nonumber
\end{align}
where the last inequality follows that $\theta^{\frac{1}{3}}\Delta^{\frac{2}{3}}L^\frac{2}{3}\max\{b^{\frac{2}{3}},\sigma^{\frac{2}{3}}\delta^\frac{1}{3}\}=O(\Delta L/\delta^{\frac{1}{3}}T^{\frac{1}{3}})$ because of \eqref{eqn:R-2}.
\end{proof}

\section{More Details of Table \ref{tab:introduction-comparison}}\label{app:more-details}
There exist mismatches between some rates listed in Table \ref{tab:introduction-comparison} and those established in literature. The mismatches exist because
\begin{enumerate}
    \item we have strengthened the vanilla rates by relaxing their restrictive assumptions (say, Double-Squeeze) or uncovering the hidden terms (say, CSER);
    \item we have extended the vanilla rates to the same setting as NEOLITHIC (say, extend MEM-SGD to non-convex and smooth setting, or transform QSGD to the distributed setting).
\end{enumerate}
With these modifications, these baseline algorithms can be compared with NEOLITHIC in a fair manner. Next we clarify each modification one by one.

\subsection{Q-SGD}\label{app:more-details-asgd}
There is a slight inconsistency between the original rate in \cite{Jiang2018ALS} and   the one listed Table \ref{tab:introduction-comparison}
due to the following reasons:

\begin{enumerate}
    \item We noticed that the rate stated in \cite[Corollary 3]{Jiang2018ALS} is not optimal following \cite[Theorem 2]{Jiang2018ALS} since the authors set the learning rate as $\theta \sqrt{M/T}$ where $\theta$ is a universal constant. The rate in \cite[Corollary 3]{Jiang2018ALS} can be slightly improved in terms of $\sigma^2$,  $b^2$, $\Delta = f(x^{(0)}) -\min_x f(x)$  by involving them into the learning rate. We manually optimize the learning rate by choosing  $\Theta((L+(T(1+\omega)L\sigma^2/\Delta M)^{1/2}+(\omega LTb^2/n\Delta)^{1/2})^{-1})$.
    \item $M$ is the total mini-batch size on all workers in Q-SGD (see the paragraph of contribution, page 2, \cite{Jiang2018ALS}).  For a fair comparison with other methods in Table \ref{tab:introduction-comparison}, We set $M=n$ to make the number of total gradient queries per iteration equivalent in all algorithms.
\end{enumerate}

\subsection{CSER}
While \cite[Corollary 1]{Xie2020CSERCS} does not have a $O(1/T)$ term in the convergence rate, it does impose another condition $T\gg n$ to the convergence statement. If $T\gg n$ holds, then $1/\sqrt{nT}\gg 1/T$ and hence the $1/T$ term is dominated by $1/\sqrt{nT}$ and thus hidden in the  notation $O(\cdot)$. However, the condition  does not appear in the convergence theorem of NEOLITHIC. To conduct a fair comparison, we have to remove condition $T\gg n$ from its convergence theorem, which thus incurs additional $O(1/T)$ term in the rate accordingly. The rate we provided in Table \ref{tab:introduction-comparison} is hence more precise than that in \cite{Xie2020CSERCS}.

% The existence of $O(1/T)$ can be verified as follows. Recall the first term in \cite[Theorem 1]{Xie2020CSERCS} as $O(\frac{1}{\eta T})$ where $\eta$ is the learning rate.  When $T$ is small (which does not satisfy $T\gg n$), it holds $\eta \leq 1/L$ according to the learning rate condition in \cite[Corollary 1]{Xie2020CSERCS}. Substituting  to $O(\frac{1}{\eta T})$ leads to $O(L/T)$. In other words, the dependence of $1/T$ is essential and cannot be erased from the rate in general. 

\subsection{Double-Squeeze}

The original rate of Double-Squeeze established in \cite{Tang2019DoubleSqueezePS} is based on an unrealistic assumption that accumulated compression errors are bounded by unknown $\epsilon$ (see \cite[Assumption 1.3]{Tang2019DoubleSqueezePS}). This makes its rate incomparable with other methods. However, we can remove the unrealistic assumption and easily derive a comparable bound where $\epsilon$ can be explicitly replaced with $O(G/\delta^2)$. We plug $O(G/\delta^2)$ into \cite[Corollary 2]{Tang2019DoubleSqueezePS} to get the rate listed in our Table \ref{tab:introduction-comparison}. 

To this end, we follow the notations of \cite{Tang2019DoubleSqueezePS} to derive explicit upper bounds for server/worker compression errors $\EE[\|\bsdelta_t\|]$ and $\EE[\|\bsdelta_t^{(i)}\|]$ which, combined with \cite[Corollary 2]{Tang2019DoubleSqueezePS}, leads to the rate in our Table \ref{tab:introduction-comparison}.
We consider compressors $Q$ (in server) and $Q_i$ (in worker $i$) utilized are $\delta$-contractive. We use  $\bsg_t^{(i)}$ to indicate local (stochastic) gradient. 

% \subsection{$\EE[\|\bsdelta_t^{(i)}\|]=O(G/\delta)$}
\textbf{Bound of local compression error: $\EE[\|\bsdelta_t^{(i)}\|]=O(G/\delta)$.}
By $\delta$-contraction and Young's inequality, we have for any $\rho>0$ that
\begin{align}
\mathbb{E}[\Vert\boldsymbol{\delta}_t^{(i)}\Vert^2]=&\mathbb{E}[\Vert\boldsymbol{v}_t^{(i)}-Q_i(\bsv_t^{(i)})\Vert^2]\nonumber\\
\leq &(1-\delta)\mathbb{E}[\Vert \bsv_t^{(i)}\Vert^2]=(1-\delta)\mathbb{E}[\Vert \bsg_t^{(i)}+\boldsymbol{\delta}_{t-1}^{(i)}\Vert^2]\nonumber\\
\leq& (1+\rho)(1-\delta)\mathbb{E}[\Vert \bsg_t^{(i)}\Vert^2]+(1+1/\rho)(1-\delta)\mathbb{E}[\Vert \boldsymbol{\delta}_{t-1}^{(i)}\Vert^2]\label{eqn:veofwq}
% \leq& (1+\rho)(1-\delta)\mathbb{E}[\Vert \mathbf{g}_t^{(i)}\Vert^2]+(1+1/\rho)(1-\delta)(1+\rho)(1-\delta)\mathbb{E}[\Vert \mathbf{g}\_{t-1}^{(i)}\Vert^2]\\
% &\quad +(1+1/\rho)^2(1-\delta)^2\mathbb{E}[\Vert \boldsymbol{\delta}\_{t-2}^{(i)}\Vert^2]\\
% \leq& \cdots \leq (1+\rho)(1-\delta)\sum_{s=1}^t  (1+1/\rho)^{t-s}(1-\delta)^{t-s}\mathbb{E}[\Vert \mathbf{g}\_s^{(i)}\Vert^2]\\
% \leq& \frac{(1+\rho)(1-\delta) G^2}{1-(1+1/\rho)(1-\delta)}
\end{align}
Iterating \eqref{eqn:veofwq} for $t, t-1, \dots,0$ and noting $\bsdelta_0^{(i)}=0$, we reach
\begin{align}
   \mathbb{E}[\Vert\boldsymbol{\delta}_t^{(i)}\Vert^2]\leq & (1+\rho)(1-\delta)\sum_{s=1}^t  (1+1/\rho)^{t-s}(1-\delta)^{t-s}\mathbb{E}[\Vert \bsg_s^{(i)}\Vert^2]\nonumber\\
\leq& \frac{(1+\rho)(1-\delta) G^2}{1-(1+1/\rho)(1-\delta)}\label{eqn:veofwq2}
\end{align}
where the last inequality holds because $\mathbb{E}[\Vert \bsg_s^{(i)}\Vert^2]\leq G^2$ for all $1\leq s\leq t$.  Here one must choose $\rho =\Omega(1/\delta)$ to avoid the explosion of the upper bound, which leads to $\mathbb{E}[\Vert\bsdelta_t^{(i)}\Vert^2]=O({G^2/\delta^2})$.
Therefore, by Jessen's inequality, we have $\EE[\|\bsdelta_t^{(i)}\|]\leq \sqrt{\EE[\|\bsdelta_t^{(i)}\|^2]}=O(G/\delta)$.

\textbf{Bound of global compression error: $\EE[\|\bsdelta_t\|]=O(G/\delta^2)$.}
% \subsection{$\EE[\|\bsdelta_t\|]=O(G/\delta^2)$}
By $\delta$-contraction and Young's inequality, we have for any $\rho>0$ that
\begin{align}
\mathbb{E}[\Vert\boldsymbol{\delta}_t\Vert^2]=&\mathbb{E}[\Vert\boldsymbol{v}_t-Q(\bsv_t)\Vert^2]\nonumber\\
\leq &(1-\delta)\mathbb{E}[\Vert \bsv_t\Vert^2]=(1-\delta)\mathbb{E}\left[\left\Vert \frac{1}{n}\sum_{i=1}^n Q_i(\bsv_t^{(i)})+\boldsymbol{\delta}_{t-1}\right\Vert^2\right]\nonumber\\
\leq &(1-\delta)(1+\rho)\mathbb{E}\left[\left\| \frac{1}{n}\sum_{i=1}^n Q_i(\bsv_t^{(i)})\right\Vert^2\right]+(1+1/\rho)(1-\delta)\EE[\|\boldsymbol{\delta}_{t-1}\Vert^2].\label{eqn:vmoedmsd}
\end{align}
Again by Young's inequality and $\delta$-contraction, we have
\begin{align}
    \mathbb{E}\left[\left\| \frac{1}{n}\sum_{i=1}^n Q_i(\bsv_t^{(i)})\right\Vert^2\right]\leq& 2\mathbb{E}\left[\left\Vert \frac{1}{n}\sum_{i=1}^n Q_i(\bsv_t^{(i)})-\bsv_t^{(i)}\right\Vert^2\right] +2\mathbb{E}\left[\left\Vert \frac{1}{n}\sum_{i=1}^n \bsv_t^{(i)}\right\Vert^2\right]\nonumber\\
    \leq& \frac{2-\delta}{n}\sum_{i=1}^n\EE[\|\bsv_t^{(i)}\|^2]=O(G^2/\delta^2)\label{eqn:vjoingbdf}
\end{align}
where the last identity is because the upper bound of $\EE[\|\bsv_t^{(i)}\|^2]$ can be obtained by following the derivation in \eqref{eqn:veofwq} and \eqref{eqn:veofwq2}.
Taking $\rho ={2(1-\delta)}/{\delta}=O(1/\delta)$ in \eqref{eqn:vmoedmsd} and using \eqref{eqn:vjoingbdf}, we reach an inequality taking a form like 
\begin{equation}\label{eqn:vjeogewq}
    \mathbb{E}[\Vert\boldsymbol{\delta}_t\Vert^2]\leq (1-\delta/2)\EE[\|\boldsymbol{\delta}_{t-1}\Vert^2]+O(G^2/\delta^3).
\end{equation}
Iterating \eqref{eqn:vjeogewq} similarly, we easily reach $\mathbb{E}[\Vert\boldsymbol{\delta}_t\Vert^2]=O(G^2/\delta^4)$ and thus $\mathbb{E}[\Vert\boldsymbol{\delta}_t\Vert]=O(G/\delta^2)$.

\subsection{MEM-SGD}
We notice that only the rate for strongly convex problems is established in the original MEM-SGD paper \cite{Stich2018SparsifiedSW}. 
% Since MEM-SGD is a milestone communication compression algorithm with effective practical performance, we list it as an important baseline. 
To compare it fairly with NEOLITHIC, we derive its convergence rate in the non-convex setting by ourselves. 

The main recursion of MEM-SGD is 
$$\bp_t^i=\eta \nabla f_i(\bx_t,\xi_t^i)+\be_t^i, \quad \bx_{t+1}=\bx_t-\frac{1}{n}\sum_{i=1}^n Q_i(\bp_t^i),\quad \be_{t+1}^i=\bp_t^i-Q_i(\bp_t^i).$$ 
The key steps of our derivation are listed as follows.

\begin{enumerate}
    \item Following the similar argument to \eqref{eqn:veofwq} and \eqref{eqn:veofwq2}, we can bound the compression error as $\mathbb{E}[\Vert\be_t^i\Vert^2]=O(\eta^2G^2/\delta^2)$. Note that here $\eta^2$ appears since compression is conducted after multiplying the learning rate $\eta$. 
    \item The recursion formula of MEM-SGD is $\by_{t+1}= \by_t- \frac{\eta}{n}\sum_{i=1}^n \nabla f_i(\bx_t,\xi_t^i)$ with $\by_t\triangleq\bx_t-\frac{1}{n}\sum_{i=1}^n\be_t^i$. In fact, one can easily check that 
    \begin{align}
        \bx_{t+1}=&\bx_t-\frac{1}{n}\sum_{i=1}^n Q_i(\bp_t^i)\nonumber\\
        =&\bx_t-\frac{1}{n}\sum_{i=1}^n (\bp_t^i-\be_{t+1}^i)=\bx_t-\frac{1}{n}\sum_{i=1}^n (\eta \nabla f_i(\bx_t,\xi_t^i)+\be_t^i-\be_{t+1}^i)\nonumber\\
        =&\by_t -\frac{\eta}{n}\sum_{i=1}^n \nabla f_i(\bx_t,\xi_t^i)+\frac{1}{n}\sum_{i=1}^n \be_{t+1}^i.\nonumber
    \end{align}
    \item Following the derivation of \eqref{eqn:gjpo1ofq2}, one can obtain
\begin{align}
    \EE[f(\by_{t+1})]-\EE[f(\by_{t})]\leq 2\eta L^2\EE[\|\by_t-\bx_t\|^2]-\frac{\eta(1-\eta L)}{2}\EE[\|\nabla f(\bx_t)\|^2]+\frac{\eta^2 L\sigma^2}{2n}.\label{eqn:nefwqfq}
\end{align}
Setting $\eta \leq \frac{1}{2L}$ such that $\frac{\eta(1-\eta L)}{2}\geq \frac{\eta}{4}$ and rearranging \eqref{eqn:nefwqfq}, we have 
\begin{align}\label{eqn:veofwqdv}
    \EE[\|\nabla f(\bx_t)\|^2]\leq& \frac{4(\EE[f(\by_{t})]-\EE[f(\by_{t+1})])}{\eta}+\frac{2\eta L\sigma^2}{n}+8L^2\EE[\|\by_t-\bx_t\|^2].
\end{align}
\item By the definition of $\by_t$ and step 1., we have
\begin{align}\label{eqn:oegnfgds}
    \EE[\|\by_t-\bx_t\|^2]\leq \frac{1}{n}\sum_{i=1}^n\EE[\|\be_t^i\|^2]=O(\eta^2 G^2/\delta^2).
\end{align}
Averaging \eqref{eqn:veofwqdv} with \eqref{eqn:oegnfgds} plugged into, we reach
\begin{align}\label{eqn:veofwqdvfds}
    \frac{1}{T}\sum_{t=0}^{T-1}\EE[\|\nabla f(\by_t)\|^2]\leq&O\left( \frac{\EE[f(\bx_{0})]-f^\star}{\eta}+\frac{\eta L\sigma^2}{n}+\frac{\eta^2 L^2G^2}{\delta^2}\right).
\end{align}
Setting the learning rate $\eta=(2L+(\frac{L\sigma^2}{n \Delta})^{1/2}+(\frac{L^2 G^2}{\delta^2\Delta})^{1/3})^{-1}$ in \eqref{eqn:veofwqdvfds} leads to the rate we listed in Table \ref{tab:introduction-comparison}.
\end{enumerate}

\subsection{Comparison with More Algorithms}
We supplement the comparison between Q-SGD \cite{Jiang2018ALS}, VR-MARINA \cite{Gorbunov2021MARINAFN}, and DASHA-MVR \cite{Tyurin2022DASHADN} in Table \ref{tab:more-comparison}. The results are compared in terms of the communication/query complexity to reach $\EE[\|\nabla f(x)\|^2]\leq \epsilon$ for a sufficiently small $\epsilon$.

Several additional comments are as follows:
\begin{enumerate}
    \item All three algorithms utilize unidirectional, unbiased, and independent compressors.
    \item All algorithms conduct an imbalanced number of compressed communications and of gradient queries. We therefore list the communication and gradient query complexity separately in Table \ref{tab:more-comparison}. The result of Q-SGD is slightly tuned by us, see the argument in Appendix \ref{app:more-details-asgd}.
    \item MARINA (with $O(1/\epsilon^2)$) and SASHA (with $O(1/\epsilon^3)$) have better communication complexity than QSGD (with $O(1/\epsilon^4)$) when $\epsilon$ is sufficiently small.
\end{enumerate}

\begin{table*}[t]
	\caption{\small Comparison between between Q-SGD \cite{Jiang2018ALS}, VR-MARINA \cite{Tyurin2022DASHADN}, and SASHA-MVR \cite{Gorbunov2021MARINAFN}. To explicitly clarify the influence of different compression strategies, we keep the stochastic gradient variance $\sigma^2$, data heterogeneity bound $b^2$
	, mini-batch size $B$ (used in \cite{Jiang2018ALS,Gorbunov2021MARINAFN}),   probability of conducting uncompressed communication $p$ (used in \cite{Tyurin2022DASHADN}), but omit
	smoothness constant $L$, and initialization gap $f(x^{(0)})-f^\star$ in the below results.
	}
	\footnotesize
	\centering
	\begin{tabular}{c  l l}
		\toprule
		\textbf{Algorithm} & \textbf{\#Communication} &  \textbf{\#Gradient Query per Worker} \\ \midrule
		Q-SGD \cite{Jiang2018ALS} & $O\left(\frac{1}{nB\epsilon^4}\left((1+\omega)\sigma^2+\omega b^2\right)\right)$ & $B\times$\#Communication \vspace{1mm} \\ 
		VR-MARINA \cite{Gorbunov2021MARINAFN} & $O\left(\frac{1}{\epsilon^2}\left(1+\sqrt{\frac{1-p}{pn}\left(\omega + \frac{1+\omega}{\max\{1,\frac{\sigma^2}{n\epsilon^2}\}}\right)}\right)\right)$ & $\max\{1,\frac{\sigma^2}{n\epsilon^2}\}\times$\#Communication \vspace{1mm} \\ 
		DASHA-MVR \cite{Tyurin2022DASHADN} & $O\left(\frac{1}{\epsilon^2}\left(1+\frac{\sigma}{n\epsilon B^{3/2}}+\frac{\sigma^2}{\epsilon^2B}\right)\right)$ & $B\times$\#Communication\vspace{1mm} \\ 
		\bottomrule
	\end{tabular}
	\label{tab:more-comparison}
    % \vspace{-8mm}
\end{table*}

\end{document}